\documentclass[letterpaper,12pt]{article}

\usepackage{amssymb,amsmath,mathabx,amsthm,amsfonts,mathrsfs,dsfont}
\usepackage{booktabs,colortbl,cases}
\usepackage{relsize}
\usepackage{dsfont}
\usepackage{url,subcaption}
\usepackage[latin1]{inputenc}
\usepackage{eufrak}
\usepackage{epsfig}
\usepackage{color}
\usepackage[noadjust]{cite}
\usepackage{mathrsfs}
\usepackage{tikz}
\usepackage{enumerate}
\usepackage[ruled,vlined]{algorithm2e}
\usepackage{psfrag}
\usepackage{graphics,adjustbox} 
\usepackage{epsfig} 
\usetikzlibrary{arrows,shapes.geometric,positioning}
\definecolor{my_GREEN}{rgb}{0,0.5,0}
\definecolor{LightCyan}{rgb}{0.88,1,1}
\definecolor{my_Gray}{rgb}{0.85,0.85,0.85}
\usepackage{caption}
\usepackage{subcaption}
\graphicspath{{figures/}{figures/MMC/}{figures/TVA/}}
\title{\LARGE \bf
Biomimetic Algorithms for Coordinated Motion: Theory and Implementation
}

\author{Udit Halder$^{1}$ and Biswadip Dey$^{2}$
\thanks{This paper is an extended version of \cite{ICRA_version}.}
\thanks{This research was supported in part by the Air Force Office of Scientific Research under AFOSR Grant FA9550-10-1-0250. Experimental evaluation of the control laws was carried out in the Intelligent Servosystems Laboratory, on a physical test-bed system for synthesis of collective behavior from fundamental building blocks, supported by a FY2012 DURIP Grant from the AFOSR (FA2386-12-1-3002).}
\thanks{$^{1}$Udit Halder is with the Department of Electrical and Computer Engineering, University of Maryland, College Park, MD, USA. {\tt\small udit@umd.edu}}%
\thanks{$^{2}$Biswadip Dey is with the Department of Electrical and Computer Engineering, Institute for Systems Research, University of Maryland, College Park, MD, USA. {\tt\small biswadip@umd.edu}}%
}

\newtheorem{theorem}{Theorem}[section]

\newtheorem{proposition}[theorem]{Proposition}

\newtheorem{defn}[theorem]{Definition}

\newtheorem{remark}[theorem]{Remark}

\definecolor{myCOLOR1}{RGB}{0,0,255}

\begin{document}

\maketitle
\thispagestyle{empty}

\begin{abstract}
Drawing inspiration from flight behavior in biological settings (e.g. territorial battles in dragonflies, and flocking in starlings), this paper demonstrates two strategies for coverage and flocking. Using earlier theoretical studies on mutual motion camouflage, an appropriate steering control law for area coverage has been implemented in a laboratory test-bed equipped with wheeled mobile robots and a Vicon high speed motion capture system. The same test-bed is also used to demonstrate another strategy (based on local information), termed topological velocity alignment, which serves to make agents move in the same direction. The present work illustrates the applicability of biological inspiration in the design of multi-agent robotic collectives. 
\end{abstract}
%
\textbf{Index Terms} - Path Planning for Multiple Mobile Robots or Agents, Autonomous Navigation, Distributed Robot Systems, Wheeled Robots
%
%
%
\section{Introduction}
\label{sec:1_Intro}
%
The last few decades have witnessed an increase in research efforts towards uncovering mechanisms behind pursuit behavior \cite{DragonFly_OLBERG, Mizutani_MC_Nature, ghosh_PLOS, chen_JEB} and collective motion \cite{PigeonPaper, Ballerini_TOPO, Scale, Fish_School_JTB} in nature. Parallel to this stream of endeavors, a number of mathematical models and appropriate feedback mechanisms \cite{MC_PSK_prs, MC_3d_Vishwa, mischiati2010motion, mischiati2012dynamics, Broucke_TAC, Galloway_PRS_13, Vicsek_Model, Justh_PSK_SCL04, Cavagna_WAVE} have been introduced to bring these ideas from natural settings to unmanned vehicular systems. Moreover, some groups in the robotics community have performed successful implementation of these concepts \cite{Srinivasan_OPTIC_FLOW, Vicsek_IROS_2014}, and thereby demonstrated the power of a bio-inspired approach towards synthesizing collective motion. This current work of ours is similar in spirit, and provides an indoor demonstration of two strategies for coverage and flocking behavior.

In this paper, we have implemented mutual motion camouflage (MMC) \cite{mischiati2012dynamics} on a mobile robot test-bed. Existing literature on dragonflies \cite{DragonFly_FieldBook} provides qualitative analysis of territorial battles, wherein the trajectories display spiraling motion consistent with the theoretical predictions \cite{Mischiati_IFAC_2011}. This particular bio-inspired control algorithm inherits an appealing coverage property through the mechanism of space filling curves, and our implementations are able to reproduce coverage patterns similar to the predicted ones.

Although there has been a long history of control algorithms for flocking, almost every model of collective motion predicts diffusive transport of information. But, contrary to the existing models, recent findings \cite{Cavagna_WAVE} from starling flocks suggest that directional information within a flock propagates with an almost constant speed, and this linear growth of information can be explained by models with wave-like aspects. In the later part of this paper, we have introduced a control strategy, called topological velocity alignment (TVA), which conforms to this criterion and can explain how information about local neighbors can influence the agents in a flock to align their headings in a single common direction. Moreover, an ongoing research work has shown that empirically observed curvature values in starling flocks carry a strong resemblance with the curvature values predicted by the steering control laws associated with TVA. Hence it seems reasonable to use TVA strategy for collective motion synthesis. Furthermore, our implementation results in real robots have shown that reduction in neighborhood size and external perturbation (similar to predator attack) can split a flock into smaller subgroups.

\begin{figure}[t]
\begin{center}
  \includegraphics[width=0.45\textwidth]{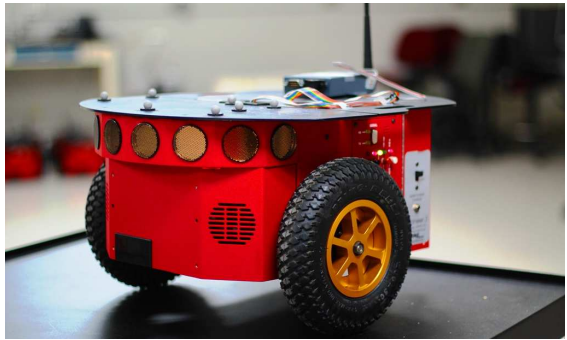}
  \caption{Mobile robot based experimental platform (Pioneer 3 DX) with two-wheel differential and caster.} 
  \label{Robot_P3_DX}
\end{center}
\end{figure}
\section{Self-steering Particle Model}
\label{sec:2_Agent}
%

By considering an agent to be a unit-mass, self-steering particle, we use natural Frenet frame equations \cite{bishop1975there} to describe its motion. This approach to describe an agent's motion is based on the fact that, while the unit tangent vector $\mathbf{x}_i$ for a given trajectory $\mathbf{r}_i$ is unique, we can choose two unit vectors $(\mathbf{y}_i, \mathbf{z}_i)$ such that $\{\mathbf{x}_i, \mathbf{y}_i, \mathbf{z}_i\}$ defines a right-handed orthogonal frame. The evolution of this frame along the agent's trajectory is given by
\begin{equation}
\begin{aligned}
\dot{\mathbf{r}}_i &= \nu_i\mathbf{x}_i 
\\
\dot{\mathbf{x}}_i &= \nu_i\left(u_i\mathbf{y}_i + v_i\mathbf{z}_i\right) 
\\
\dot{\mathbf{y}}_i &= -\nu_iu_i\mathbf{x}_i 
\\
\dot{\mathbf{z}}_i &= -\nu_iv_i\mathbf{x}_i, 
\end{aligned}
\label{natural_frenet_3D}
\end{equation}
where $\nu_i$ represents speed and ($u_i$,$v_i$) are the natural curvatures. In a similar way, by restricting the motion of an agent to a planar setting, we can model its motion as
\begin{equation}
\begin{aligned}
\dot{\mathbf{r}}_i &= \nu_{i}\mathbf{x}_i
\\
\dot{\mathbf{x}}_i &= \nu_{i} u_{i} \mathbf{y}_i
\\
\dot{\mathbf{y}}_i &= -\nu_{i} u_{i} \mathbf{x}_i,
\end{aligned}
\label{natural_frenet_2D}
\end{equation}
where $\mathbf{x}_i$ denotes normalized velocity of the $i$-th agent and $\mathbf{y}_i$ is the orthogonal rotation of $\mathbf{x}_i$ in the counter-clockwise direction. It should be noted here that this way of modeling a trajectory requires only twice-differentiability of the trajectory. 

\begin{remark}
As we use Hilare type \cite{AutoBot_Book} mobile robots (Fig~\ref{Robot_P3_DX}) as a test-bed for control algorithms, \eqref{natural_frenet_2D} provides a natural choice to model their motion.
\end{remark}
\section{Mutual Motion Camouflage (MMC)}
%
Here we consider the \textit{mutual} interaction between two agents each applying the same pursuit law, while perceiving the other one as a target. As the dynamics of MMC in a planar setting has been studied earlier \cite{mischiati2012dynamics}, we just reiterate some key results in order to have a comprehensive framework. Allowing different speeds for the agents, we begin with the following symmetry: 
\begin{equation} \label{eq:symmetry}
u_{1}\nu_{1} = u_{2}\nu_{2} = u.
\end{equation}
Then the dynamics of the relative motion vectors, namely $\mathbf{r} = \mathbf{r}_1 - \mathbf{r}_2$, $\mathbf{g} = \dot{\mathbf{r}} = \nu_1 \mathbf{x}_1 - \nu_2 \mathbf{x}_2$ and $\mathbf{h} = \mathbf{g}^\bot = \dot{\mathbf{r}}^{\bot}$, can be expressed as
\begin{equation} \label{eq:rel_model}
\begin{split}
&\dot{\mathbf{r}} = \mathbf{g}\\
&\dot{\mathbf{g}} =  u \mathbf{h}\\
&\dot{\mathbf{h}} = -u \mathbf{g}.
\end{split} 
\end{equation} 
Now we introduce three scalar shape variables defined as $\rho = |\mathbf{r}|$, $\gamma =(\mathbf{r} \cdot \mathbf{g})/{|\mathbf{r}|}$ and $\lambda=(\mathbf{r}\cdot \mathbf{h})/{|\mathbf{r}|}$. Then, according to \cite{MC_PSK_prs}, we have
\begin{equation} \label{eq:curvature_MMC}
u 
= 
-\mu\left(\frac{\mathbf{r}}{|\mathbf{r}|} \cdot \dot{\mathbf{r}}^{\bot} \right)
=
-\mu\left(\frac{\mathbf{r}}{|\mathbf{r}|} \cdot \mathbf{h} \right) 
=
-\mu \lambda,
\end{equation}
where, $ \mu > 0 $ denotes the feedback gain. As shown earlier, the dynamics of relative motion \eqref{eq:rel_model} can be reduced to yield a second order dynamics given by
\begin{equation} \label{eq:rho-gamma}
\begin{split}
& \dot{\rho} = \gamma \\
& \dot{\gamma} = \left( 1/\rho - \mu \right) \left(\delta^2 - \gamma^2 \right),
\end{split}
\end{equation} 
where, $\delta = |\mathbf{g}| = |\mathbf{h}|$ is conserved along any trajectory of \eqref{eq:rel_model}. As detailed in the original work \cite{mischiati2012dynamics}, individual trajectories can be reconstructed from the solutions of \eqref{eq:rho-gamma}. Moreover, the solutions of the reduced dynamics \eqref{eq:rho-gamma} constitute level sets for another conserved quantity, defined as
\begin{equation}\label{eq:conserved}
E(\rho,\gamma) = \rho^2(\delta^2 - \gamma^2)e^{-2\mu\rho} = E(\rho_0, \gamma_0).
\end{equation}

\subsection{Stabilization through Dissipation} \label{sec:theory_MMC_dissipation}
As it will be described in section~\ref{sec:4_Result}, implementation of the original MMC law \eqref{eq:curvature_MMC} in a laboratory environment did not result in a satisfactory performance. The trajectories diverged from the theoretical predictions, and the errors kept building up (Fig~\ref{fig:mmc_traj}). This type of behavior can be attributed to the lack of damping in the reduced dynamics \eqref{eq:rho-gamma}. As a result, every small error was kept unchecked, and accumulation of these errors gave rise to a poor performance of the MMC feedback law. Based on this observation, we considered a modified version of the feedback law which introduced a dissipation term. As $E(\rho, \gamma)$ was not staying constant during an implementation of MMC feedback law \eqref{eq:curvature_MMC}, we added a dissipative term in the feedback control to counter-act any deviation from the predicted trajectories. The resultant control law can be expressed as
\begin{equation} \label{eq:dissipation}
u_{tot} = u + u_{dis} = - \mu \lambda + k_d \lambda \gamma \big( E(\rho, \gamma) - E_d \big),
\end{equation}      
where $E_d$ is set as the initial value of the conserved quantity $E(\rho, \gamma)$ ($E_d = E(\rho_0, \gamma_0)$). Earlier research \cite{mischiati2010motion} has shown that the modified control law \eqref{eq:dissipation} with $k_d > 0$ makes the periodic orbit (with energy $E_d$) orbitally asymptotically stable, and the corresponding domain of attraction is characterized by $\{(\rho, \gamma): \rho > 0, -\delta < \gamma < \delta, (\rho, \gamma) \neq (1/\mu, 0)\}$.
\section{Topological Velocity Alignment (TVA)}
%
Here we formalize the strategy of topological velocity alignment (TVA), and assume that each member in a group of $n$-agents uses this strategy to move together while keeping its heading parallel to the neighborhood center of mass velocity. Letting $\mathcal{N}_i$ denote the neighborhood of the $i$-th agent, the center of mass (COM) velocity of this neighborhood is given by
\begin{equation}
\mathbf{v}_{_{COM}}
=
\frac{1}{|\mathcal{N}_i|}\sum_{j \in \mathcal{N}_i}\nu_j\mathbf{x}_j, 
\label{V-COM}
\end{equation}
where $|\mathcal{N}_i|$ represents the number of neighbors influencing the $i$-th agent. Next, by assuming that $\mathbf{v}_{_{COM}}$ does not vanish to zero, we define the direction of the center of mass motion as
\begin{equation}
\mathbf{x}_{\mathcal{N}_i} 
= \frac{\mathbf{v}_{_{COM}}}{|\mathbf{v}_{_{COM}}|}.
\label{T_N}
\end{equation}
It should be noted that $\mathbf{x}_{\mathcal{N}_i}$ is \textit{not well-defined over a thin set in the state space}. As the chance of getting into the thin set is \textit{essentially zero}, we can overlook this situation for all practical purposes. Now we introduce a contrast function
\begin{equation}
\Theta_i 
= \frac{1}{2} \left(\mathbf{x}_{\mathcal{N}_i} - \mathbf{x}_i \right) \cdot \left(\mathbf{x}_{\mathcal{N}_i} - \mathbf{x}_i \right) 
= 1 - \mathbf{x}_i \cdot \mathbf{x}_{\mathcal{N}_i},
\label{CoSt}
\end{equation}
as a quantitative measure for the misalignment between the heading of an agent and the direction of motion of its neighborhood center of mass. Clearly, this contrast function ($\Theta_i$) assumes its minimum value ($= 0$) whenever the $i$-th agent's velocity is aligned with its neighborhood center of mass velocity, and increases monotonically with increase in the misalignment between them. Thus, $\Theta_i$ can be interpreted as a measure of departure from our goal of achieving alignment.

Next, by assuming a non-zero velocity for the neighborhood center of mass ($\mathbf{v}_{_{COM}} \neq 0$), we propose a control law
\begin{equation}
\begin{aligned}
u_i 
&= 
\mu\left(\frac{\mathbf{x}_{\mathcal{N}_i} \cdot \mathbf{y}_i}{\nu_i}\right) 
\\
v_i 
&=
\mu\left(\frac{\mathbf{x}_{\mathcal{N}_i} \cdot \mathbf{z}_i}{\nu_i}\right),
\end{aligned}
\label{TVA_Rule}
\end{equation}
where $\mu > 0$ denotes a positive gain, and $\mathbf{y}_i$, $\mathbf{z}_i$ carry their usual meaning. Alternatively, lateral acceleration for this choice of control laws \eqref{TVA_Rule} can be expressed as
\begin{eqnarray}
\textrm{a}_i^{{lat}} 
=
\mu \nu_i \big( \mathbf{x}_{\mathcal{N}_i} - \big[\mathbf{x}_{\mathcal{N}_i} \cdot \mathbf{x}_i \big] \mathbf{x}_i \big),
\label{lat-accln}
\end{eqnarray}
and this provides a physical intuition behind \eqref{TVA_Rule} as the lateral acceleration is proportional to the projection of the normalized velocity of its neighborhood center of mass onto the transverse of its own direction of motion.

\begin{remark}
Earlier works \cite{Justh03steeringlaws,Justh_PSK_SCL04} have considered a very similar form of this control law with three components for attraction (while the agents are far away), repulsion (to avoid collision) and velocity alignment. However, the control law introduced in this current study considers only velocity alignment, but extends the scope from a planar setting to a three dimensional environment. Moreover it relaxes the assumption on uniform speed of the collective by allowing the agents non-uniform and time-varying speed profiles. This relaxation plays an important role in the context of applying this control law to a group of heterogeneous agents.
\end{remark}

%
%
%
%
%

%
%
%
%
%
\subsection{Special Case: A planar 2-agent system}
As analysis of an $n$-agent system with neighborhood defined as the set of $K$-nearest neighbors poses hard challenges, we begin by considering a special case, namely a two-agent system with the motion restricted to a planar setting. Clearly, the neighborhood center of mass never loses speed in this case (as the neighborhood is nothing but the other agent).
%

%
%
\subsubsection{State space and its reduction onto the shape space}
\label{sec:SS_2agent}
By modeling the dynamics of an agent using natural Frenet frame equations restricted to a planar setting \eqref{natural_frenet_2D}, its position and the frame vectors can be packed inside a $3 \times 3$ matrix $g_i \in SE(2)$, the Lie group of rigid motions in a plane. By excluding collocation of the agents, we define the state space of the system as
\begin{equation}
\mathcal{M}_{state} = \Big\{ g_1,g_2 \in SE(2) \times SE(2) \Big| g_1e_3 \neq g_2e_3 \Big\} \label{2agent_state_space}
\end{equation}
where $e_3 = [0\;0\;1]^T$ and $g_i = \left[\begin{array}{ccc} \mathbf{x}_i & \mathbf{y}_i & \mathbf{r}_i \\ 0 & 0 & 1 \end{array}\right]$. 

In terms of this Lie-group formulation the system dynamics can be represented as 
\begin{equation}
\dot{g}_i = g_i\xi_i(u_i) = g_i \nu_i \big( A_1 u_i + A_2 \big) \label{Lie_grp_dyna_2}
\end{equation}
where $A_1$ and $A_2$ represent infinitesimal generators of planar rotation and translation, respectively. Moreover $A_1$ and $A_2$ can generate the Lie-algebra under bracketing. As we are interested in steering laws which leave our system dynamics invariant under rigid motion, we can formulate a reduction to the \textit{shape space}, a quotient manifold $\mathcal{M}_{state}/SE(2)$ of relative positions and velocities of the agents. We define $g \in SE(2)$ as
\begin{equation}
g 
= 
g_1^{-1}g_2 
= 
\left[\begin{array}{ccc} 
\mathbf{x}_1\cdot\mathbf{x}_2 & \mathbf{x}_1\cdot\mathbf{y}_2 & \mathbf{x}_1\cdot\mathbf{r}\\ 
\mathbf{y}_1\cdot\mathbf{x}_2 & \mathbf{y}_1\cdot\mathbf{y}_2 & \mathbf{y}_1\cdot\mathbf{r}\\
0 & 0 & 1
\end{array}\right],
\label{Shape_LIE_grp}
\end{equation}
where $\mathbf{r} \triangleq \mathbf{r}_2 - \mathbf{r}_1$ denotes the baseline vector. Now the shape space for the planar two-agent system can be defined as 
\begin{equation}
\mathcal{M}_{shape} = \Big\{ g \in SE(2) \Big| g_{13}^2 + g_{23}^2 \neq 0 \Big\}.
\label{2agent_shape_space}
\end{equation}
 Moreover, $g$ assumes a left-invariant dynamics on $SE(2)$ as
\begin{equation}
\dot{g} = g\xi
\label{Shape_DyNa_LIE_grp}
\end{equation}
where $\xi = \xi_2(u_2) - g^{-1}\xi_1(u_1)g \in \mathfrak{se}(2)$, and the proposed control law \eqref{TVA_Rule} depends only on the shape variable $g$ as
\begin{equation}
u_i = \mu\left[\frac{g_{\bar{i}i}}{\nu_i}\right] 
, \quad i \in \{1,2\}, \quad \bar{i} = \{1,2\} \setminus \{i\}.
\label{fbLAW_shape}
\end{equation}
Therefore from \cite{Justh_PSK_SCL04} it can be concluded that the reduced dynamics \eqref{Shape_DyNa_LIE_grp} evolves on the shape space $\mathcal{M}_{shape}$.

%
%
\begin{figure}[h]
\psfrag{r}[][][1]{$\qquad \mathbf{r} = \mathbf{r}_2 - \mathbf{r}_1$}
\psfrag{y1}[][][1]{$\mathbf{y}_1$} 
\psfrag{x1}[][][1]{$\mathbf{x}_1$} 
\psfrag{r1}[][][1]{$\mathbf{r}_1$}
\psfrag{y2}[][][1]{$\mathbf{y}_2$}
\psfrag{x2}[][][1]{$\mathbf{x}_2$}
\psfrag{r2}[][][1]{$\qquad\mathbf{r}_2$} 
\psfrag{t1}[][][1]{$\theta_1$} 
\psfrag{t2}[][][1]{$\;\theta_2$}
\psfrag{phi1}[][][1]{$\psi$} 
\psfrag{phi2}[][][1]{$\phi$}
\psfrag{psi}[][][1]{$\vartheta$} 
\begin{center}
  \includegraphics[width=0.65\textwidth]{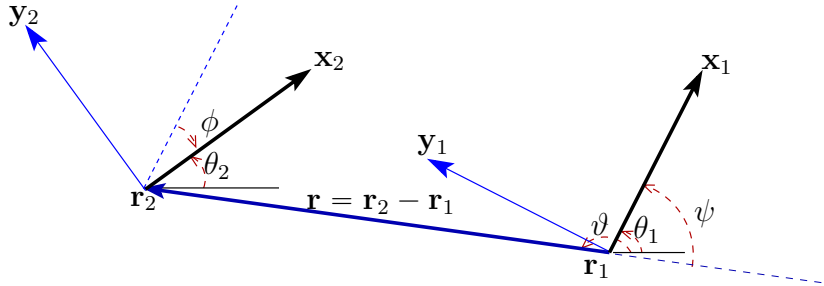}
  \caption[Illustration of scalar shape variables for a 2-agent system]{Illustration of scalar shape variables ($\rho,\psi,\phi$) used to parametrize the shape space $\mathcal{M}_{shape}$.} 
  \label{Scalar_Shapes}
\end{center}
\end{figure}

%
%
\subsubsection{Shape dynamics}
\label{sec:SD_2agent}
Now, through polar parametrization, we introduce some geometrically meaningful scalar variables to parametrize the shape space. By defining
\begin{equation}
\begin{aligned}
\mathbf{r} &= \mathbf{r}_2 - \mathbf{r}_1 = \rho \textrm{e}^{\mathrm{i}\vartheta}
\\
\mathbf{x}_1 &= \textrm{e}^{\mathrm{i}\theta_1} 
\\
\mathbf{x}_2 &= \textrm{e}^{\mathrm{i}\theta_2},
\end{aligned}
\label{base_polar}
\end{equation}
we introduce $\psi (= \pi - \vartheta + \theta_1)$ and $\phi (= \theta_1 - \theta_2)$ to represent the relative orientation of the velocity vectors. Clearly, $\psi$ represents the relative orientation of $\mathbf{x}_1$ with respect to the baseline vector $\mathbf{r}$, and $\phi$ represents the misalignment in velocity directions. From \eqref{Shape_LIE_grp} one can notice that $g \in \mathcal{M}_{shape}$ can be represented in terms of these scalar shape variables as
\begin{equation}
g
=
\left[\begin{array}{ccc} 
 \cos\phi & \sin\phi & -\rho\cos\psi\\ 
-\sin\phi & \cos\phi &  \rho\sin\psi\\
0 & 0 & 1
\end{array}\right].
\label{Shape_scalar_LIE}
\end{equation}
Then, by a straightforward calculation, one can show that the shape dynamics are given in terms of the scalar variables by
\begin{align}
\dot{\rho} 
&= 
\nu_1\cos\psi - \nu_2\cos(\psi - \phi) 
\nonumber \\
\dot{\psi} 
&= 
\nu_1 u_1 - \frac{1}{\rho}\big[ \nu_1 \sin\psi - \nu_2 \sin(\psi - \phi) \big] 
\label{shape_dynamics_CL}\\
\dot{\phi} 
&= 
\nu_1 u_1 - \nu_2 u_2.
\nonumber
\end{align}

%
%
\subsubsection{Analysis of topological velocity alignment}
Here we consider a particular context of the two-agent planar system wherein each agent employs the strategy for topological velocity alignment (TVA). In terms of the original state variables the contrast functions take the form 
\begin{equation}
\Theta_i = \frac{1}{2} \left(\mathbf{x}_{\bar{i}} - \mathbf{x}_i\right) \cdot \left(\mathbf{x}_{\bar{i}} - \mathbf{x}_i\right), \quad i \in \{1,2\} \label{Cntrst_2agnt}
\end{equation}
and the $i$-th agent is declared to attain TVA if $\Theta_i = 0$. From \eqref{Cntrst_2agnt} one can immediately notice equality of the contrast functions for both agents, and hence we define a common contrast function $\Theta = \Theta_1 = \Theta_2$. Clearly, this common contrast function can be represented in terms of scalar shape variables as
\begin{equation}
\Theta
=
1 - \cos\phi,
\label{GAMMA_shape}
\end{equation}
and hence we can conclude that $\Theta = 0$ \textit{if and only if} $\phi = 0$.

Next, for this two-agent planar system we define a 2-dimensional TVA manifold $\mathcal{M}_{TVA} \subset \mathcal{M}_{shape}$ as
\begin{equation}
\mathcal{M}_{TVA}
= 
\Big\{\rho, \psi, \phi \in \mathcal{M}_{shape} \Big| \phi = 0\Big\}.
\end{equation}
Moreover, from \eqref{fbLAW_shape} one can notice that the feedback law can be expressed in terms of shape variables, taking the form
\begin{equation}
\begin{array}{rcl}
u_1 & = & - \left(\frac{\mu}{\nu_1}\right)\sin\phi \\
u_2 & = & \quad  \left(\frac{\mu}{\nu_2}\right)\sin\phi.
\label{curvature_shape}
\end{array}
\end{equation}
Clearly, the steering control becomes identically zero on the TVA manifold, and as a consequence the mismatch in velocity direction remains identically zero \eqref{shape_dynamics_CL}. Now we will formally introduce the notion of \emph{invariance}.

\begin{defn}[Invariant Manifold]
A manifold $\mathcal{M}$ is said to be invariant under the flow of a vector field $X$ on $\mathcal{M}$ if for any $m \in \mathcal{M}$, $F_t(m)\in \mathcal{M}$ for small $t > 0$, where $F_t(\cdot)$ denotes the flow of $X$. One can show that this condition is equivalent to $X$ being tangent to the manifold $\mathcal{M}$.
\label{Inv_Mani_Defn}
\end{defn}

If both agents employ a steering control of the form \eqref{curvature_shape}, the closed loop dynamics for a two-agent planar system can be represented as
\begin{align}
\dot{\rho} 
&= 
\nu_1\cos\psi - \nu_2\cos(\psi - \phi) 
\nonumber \\
\dot{\psi} 
&= 
- \mu \sin\phi - \frac{1}{\rho}\big[ \nu_1 \sin\psi - \nu_2 \sin(\psi - \phi) \big] 
\label{shape_dynamics_CL_TVA}\\
\dot{\phi} 
&= 
-2\mu \sin\phi.
\nonumber
\end{align}
We should note that prohibition on collocation, i.e. $\rho > 0$, is not enforced by these dynamics. Now we state the following proposition associated with the closed loop shape behavior.
\begin{proposition}
The topological velocity alignment manifold $\mathcal{M}_{TVA} \subset \mathcal{M}_{shape}$ is invariant under the closed loop shape dynamics \eqref{shape_dynamics_CL_TVA}. Moreover if $\gamma(t)\in \mathcal{M}_{shape}$ is a trajectory of \eqref{shape_dynamics_CL_TVA} which does not have a finite escape time, and $\Theta(0) \neq 2$, then
\begin{equation}
\Theta(t) \rightarrow 0 \quad \textrm{as} \quad t \rightarrow \infty,
\label{attract_M_FL}
\end{equation}
i.e. $\gamma(t)$ converges asymptotically to $\mathcal{M}_{TVA}$.
\label{M_TVA_inv_2D}
\end{proposition}
\begin{proof}
From \eqref{GAMMA_shape} and \eqref{shape_dynamics_CL_TVA} we have
\begin{equation}
\dot{\Theta} 
= \dot{\phi}\sin\phi
=  -2\mu \sin^2\phi 
= -2\mu \Theta \big(2 - \Theta\big)
\label{Theta_dot_shape}
\end{equation}
and hence $\mathcal{M}_{TVA} \subset \mathcal{M}_{shape}$ is an invariant manifold under the closed loop shape dynamics. This equation \eqref{Theta_dot_shape} also implies that the agents will keep on moving in the opposite direction ($\Theta(t) = 2$) if their initial directions were opposite to each other ($\Theta(0) = 2$). 

Moreover, by assuming  $\Theta(0) \in (0,2)$, we have 
\begin{displaymath}
\Theta(t)
= \frac{2e^{-4\mu t}}{C  + e^{-4\mu t}}
\end{displaymath}
where the constant $C$ is defined as $C = \frac{2}{\Theta(0)}-1$. Since $ e^{-4\mu t} \rightarrow 0$ as $t \rightarrow \infty$, we have $\Theta(t) \rightarrow 0$ as $t \rightarrow \infty$.
\end{proof}

Next we focus on the restricted dynamics on $\mathcal{M}_{TVA}$, and by substituting $\phi = 0$, \eqref{shape_dynamics_CL_TVA} yields
\begin{equation}
\begin{aligned}
\dot{\rho} 
&= 
\big(\nu_1 - \nu_2\big)\cos\psi 
\\
\dot{\psi} 
&= 
- \frac{1}{\rho}\big( \nu_1 - \nu_2 \big) \sin\psi.
\end{aligned}
\label{reduced_dynamics_TVA}
\end{equation}
Now, by assuming $\nu_1 - \nu_2$ to be non-zero, the evolution of a phase plane trajectory can be described as
\begin{equation}
\frac{d\rho}{d\psi} = -\frac{\rho \cos\psi}{\sin\psi},
\label{phase_plane_SLOPE}
\end{equation}
which can then be integrated to show that the quantity $\rho(t)\sin\psi(t)$ is conserved along any trajectory on the TVA manifold ($\mathcal{M}_{TVA}$). Hence it is clear that the region \{$\rho>0,\pi>\psi>0$\} (or similarly \{$\rho>0,0>\psi>-\pi$\}) is positively invariant under the restricted dynamics \eqref{reduced_dynamics_TVA}. The corresponding phase-portrait is shown in Fig~\ref{Scalar_Shapes_PP}.
%
%
\begin{figure}[h]
\psfrag{x}[][][1]{$\rho$}
\psfrag{y}[][][1]{$\psi$} 
\begin{center}
  \includegraphics[width=0.75\textwidth]{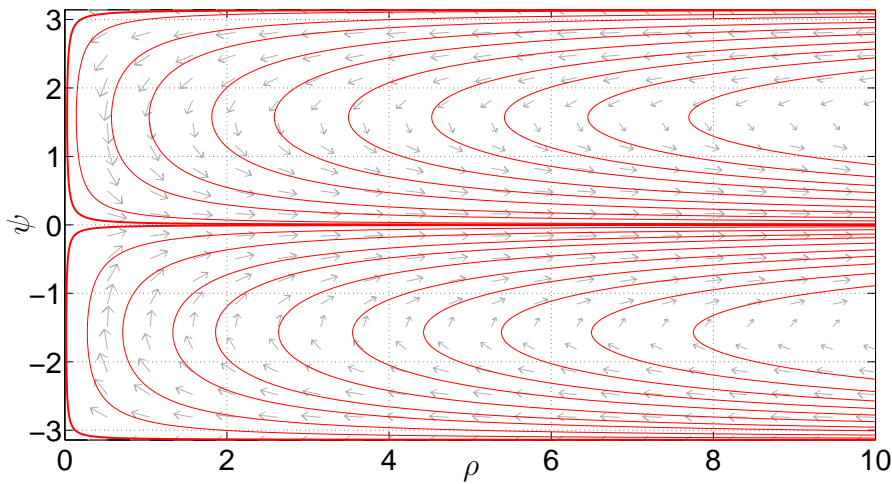}
  \caption[Phase portraits for the restricted dynamics]{Phase portraits for the restricted dynamics \eqref{reduced_dynamics_TVA}.} 
  \label{Scalar_Shapes_PP}
\end{center}
\end{figure}

%
%
%
%
%
\subsection{Algorithm for an $n$-agent system in a three-dimensional setting}
In this sub-section we focus toward TVA in its true sense. By bringing in an additional neighbor into consideration whenever $\mathbf{v}_{_{COM}}$ becomes zero, we propose an algorithmic way (Algorithm~\ref{TVA_Algo}) to implement TVA in a real system. Clearly, this provides a way to avoid ill-posedness associated with $\mathbf{v}_{_{COM}}$ being zero because non-zero speeds of individual agents ensure that considering an extra neighbor will make an otherwise zero $\mathbf{v}_{_{COM}}$ non-zero. Moreover, this approach towards flocking can easily be restricted to a planar setting, and it has been demonstrated in what follows.

\begin{algorithm}[t!]
\KwData{
Initial Time - $t_{initial}$; Final Time - $t_{final}$;
Sampling Interval - $\Delta$; Number of Agents - $n$;
Initial Position and Orientation - $\{g_i\}_{i=1}^n$; Neighborhood Size - $K$
}
\Begin{
\textbf{Initialize}: $t_{current} \longleftarrow t_{initial}$ \;
\For{$i$ = $1$ to $n$}{
\textbf{Initialize}: State - $X_i \longleftarrow g_i$ \;
}
\While{$t_{current} \leq t_{final}$}{
\For{$i$ = $1$ to $n$}{
\textbf{Define}: $\mathcal{N}_i$ - the set of $K$-nearest neighbors \;
\textbf{Compute}: Neighborhood center of mass velocity - $\mathbf{v}_{_{COM}}$ \;
\If{$\mathbf{v}_{_{COM}} = 0$}
{
\textbf{Define}: $\mathcal{N}_i$ - the set of $K+1$-nearest neighbors \;
\textbf{Compute}: Neighborhood center of mass velocity - $\mathbf{v}_{_{COM}}$ \;
}
\textbf{Compute}: Steering control - $u_i$, $v_i$ \;
}
\textbf{Implement}: Steering Control - $\{u_i,v_i\}_{i=1}^n$ \;
\textbf{Update}: State - $\{X_i\}_{i=1}^n$ \;
\textbf{Update}: Time - $t_{current} \longleftarrow t_{current} + \Delta$ \;
}
}
\caption{Topological Velocity Alignment - 3D \label{TVA_Algo}}
\end{algorithm}
\section{Implementation Results}
\label{sec:4_Result}
%

%
%
\subsection{Experimental Setup}

Our experimental test-bed is comprised of Pioneer 3 DX wheeled robots from Adept MobileRobots \cite{P3_DX_Manual}. These compact differential-drive mobile robots are equipped with reversible DC motors, high-resolution motion encoders and 19cm wheels, and the onboard computation is done via 32-bit Renesas SH2-7144 RISC microprocessor, including the P3-SH microcontroller with ARCOS. The sensors on the robot include eight forward-facing ultrasonic (sonar) sensors. ARIA \cite{ROSARIA} provides an interface for controlling and receiving data from the robot, and communication with the robot for sending control commands (\textit{forward velocity} and \textit{turning rate}) is done via 802.11-b/g/n networking. The width of the robot is 380 mm and it has a swing radius of 260 mm.\par
Algorithm implementation (i.e, feedback law computation) has been done in C++ using ROS \cite{ROS_Manual}, along with ROS-ARIA, as the interfacing robotics middleware. The experiments have been carried out in a laboratory environment equipped with a sub-millimeter accurate Vicon motion capture system \cite{Vicon_Manual}. We use a Dell workstation to run ROS, and this computer is connected to the Vicon server via a dedicated Ethernet connection.\par 
The Vicon system captures the motion of the robots and sends out the position and heading data to the computer running ROS. The control law program listens to this data, and transmits the individual \textit{turning rates} (as individual speeds are assumed to be constant). Both of these operations are carried out at a frequency of 25 Hz. As the velocity vector ($\dot{\mathbf{r}}_i^k$, with $k$ denoting the time index) is aligned along the robot heading, $\mathbf{x}_i^k$ and $\mathbf{y}_i^k$ can be directly computed from the heading data. Then, the curvature variable $u_i^k$ is evaluated from the corresponding control laws (\ref{eq:dissipation}, \ref{TVA_Rule}), and the \textit{turning rate} $\omega_i^k (i = 1,2,\ldots,n)$ (in degrees/sec) is computed as:
\begin{equation}
\omega_i^k = \left(\dfrac{180}{\pi}\right) \nu_i u_i^k.
\end{equation}

Next we will present the implementation results of the two control laws in our robotic test-bed. In this paper, we are presenting results for which the speeds of all the individual agents are same, i.e. $\nu_i = \nu_j, \forall i,j$. Though it should be noted that both control laws can be implemented with different speeds.


%
%
\subsection{Implementation of MMC}

\begin{figure*}[!h]
	\centering
	\begin{subfigure}[t]{0.48\textwidth}
		\centering
		\includegraphics[scale = 0.425, trim = 10 0 40 0, clip = false]{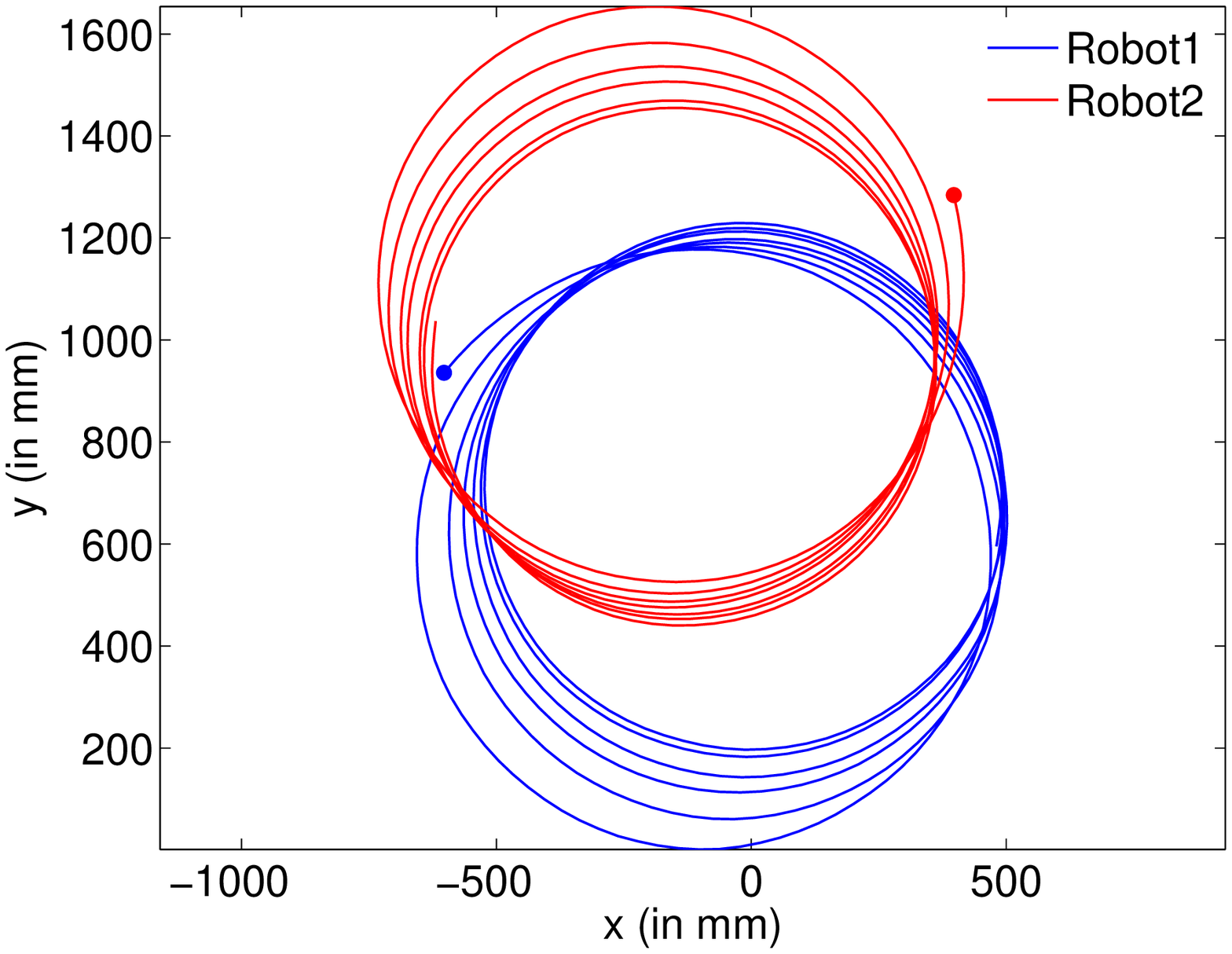}
		\caption{Trajectory (Experiment)}
	\end{subfigure}
	\hspace{0.1cm}
	\begin{subfigure}[t]{0.48\textwidth}
		\centering
		\includegraphics[scale = 0.425, trim = 10 0 40 0, clip = false]{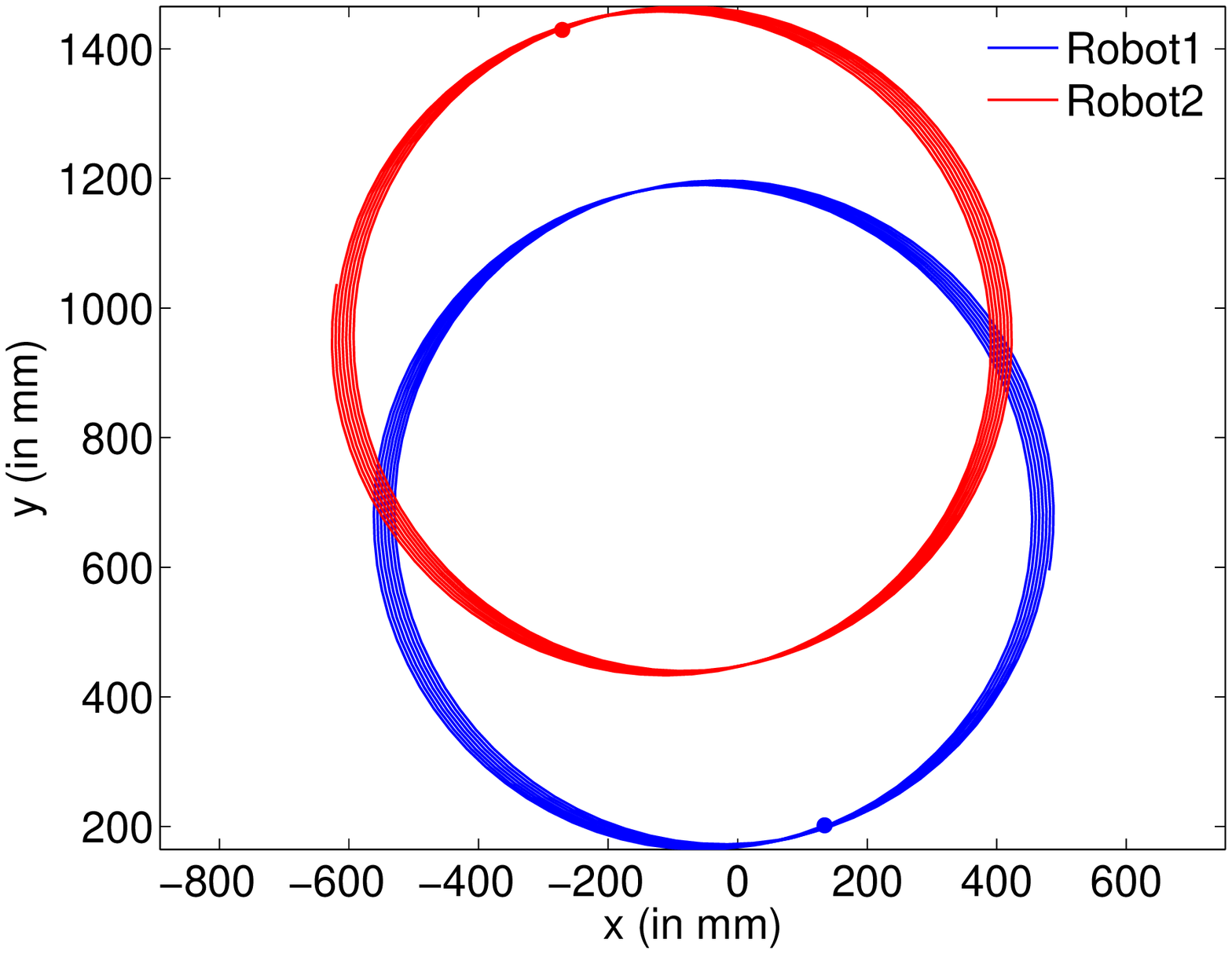}
		\subcaption{Trajectory (Simulation)}
	\end{subfigure}
	\hspace{0.1cm}
	\begin{subfigure}[t]{0.48\textwidth}
		\centering	
		\includegraphics[scale = 0.425, trim = 10 0 40 0, clip = false]{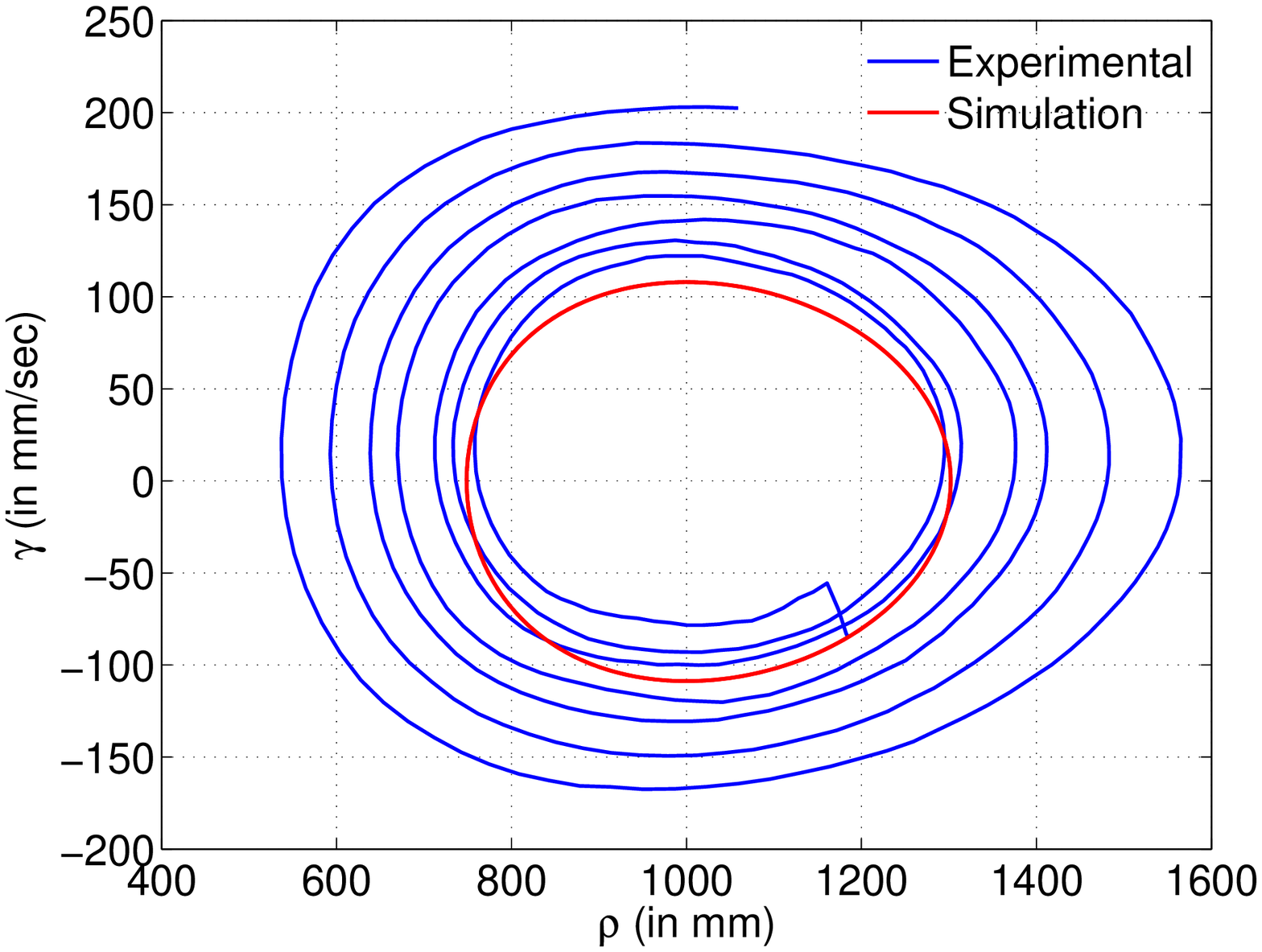}
		\caption{Phase plot}
		\label{fig:mmc_phase}
	\end{subfigure}
	\hspace{0.1cm}
	\begin{subfigure}[t]{0.48\textwidth}
		\centering	
		\includegraphics[scale = 0.425, trim = 10 0 40 0, clip = false]{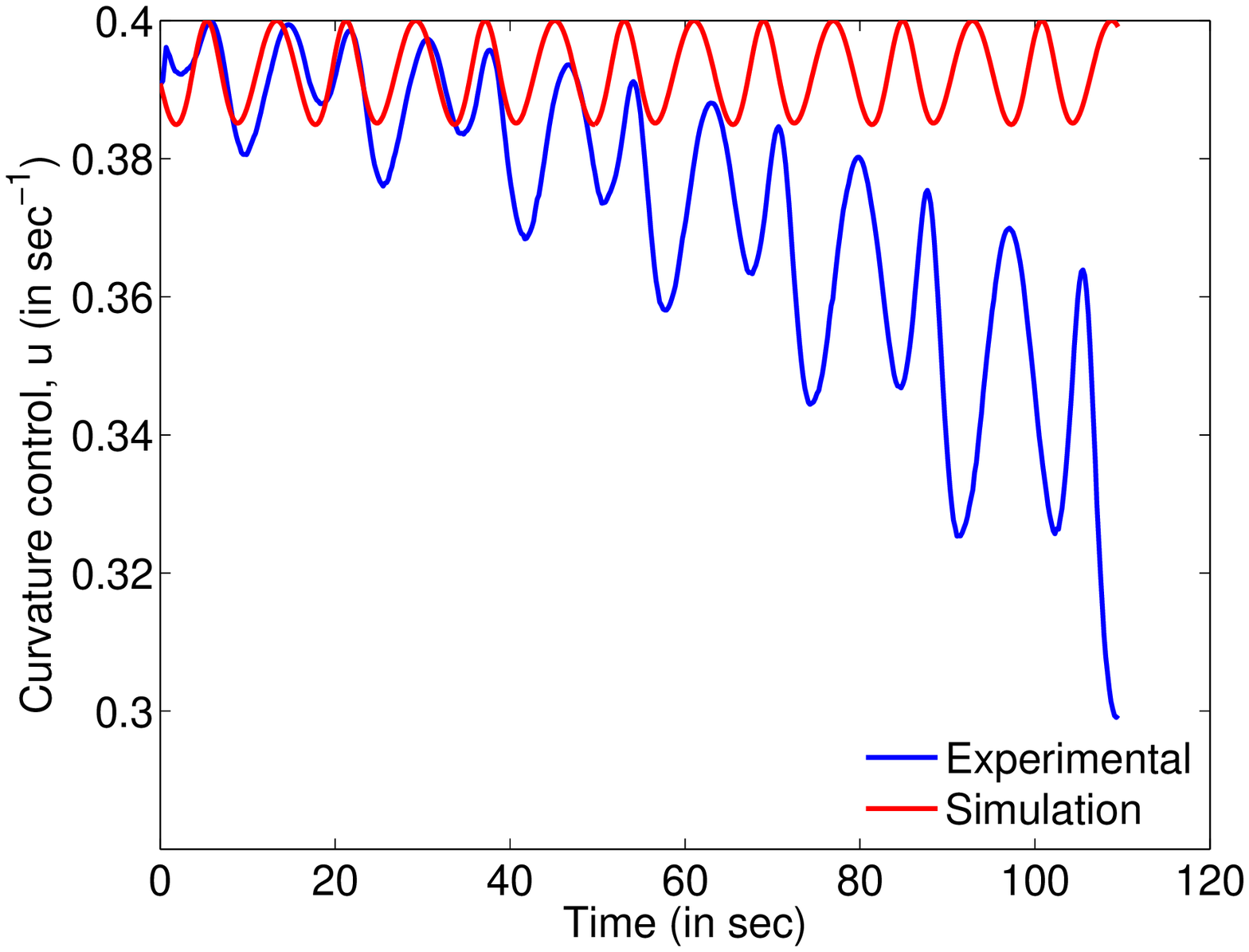}
		\caption{Curvature}
		\label{fig:mmc_curvature}
	\end{subfigure}
	\hspace{0.1cm}
	\begin{subfigure}[t]{0.48\textwidth}
		\centering	
		\includegraphics[scale = 0.425, trim = 10 0 40 0, clip = false]{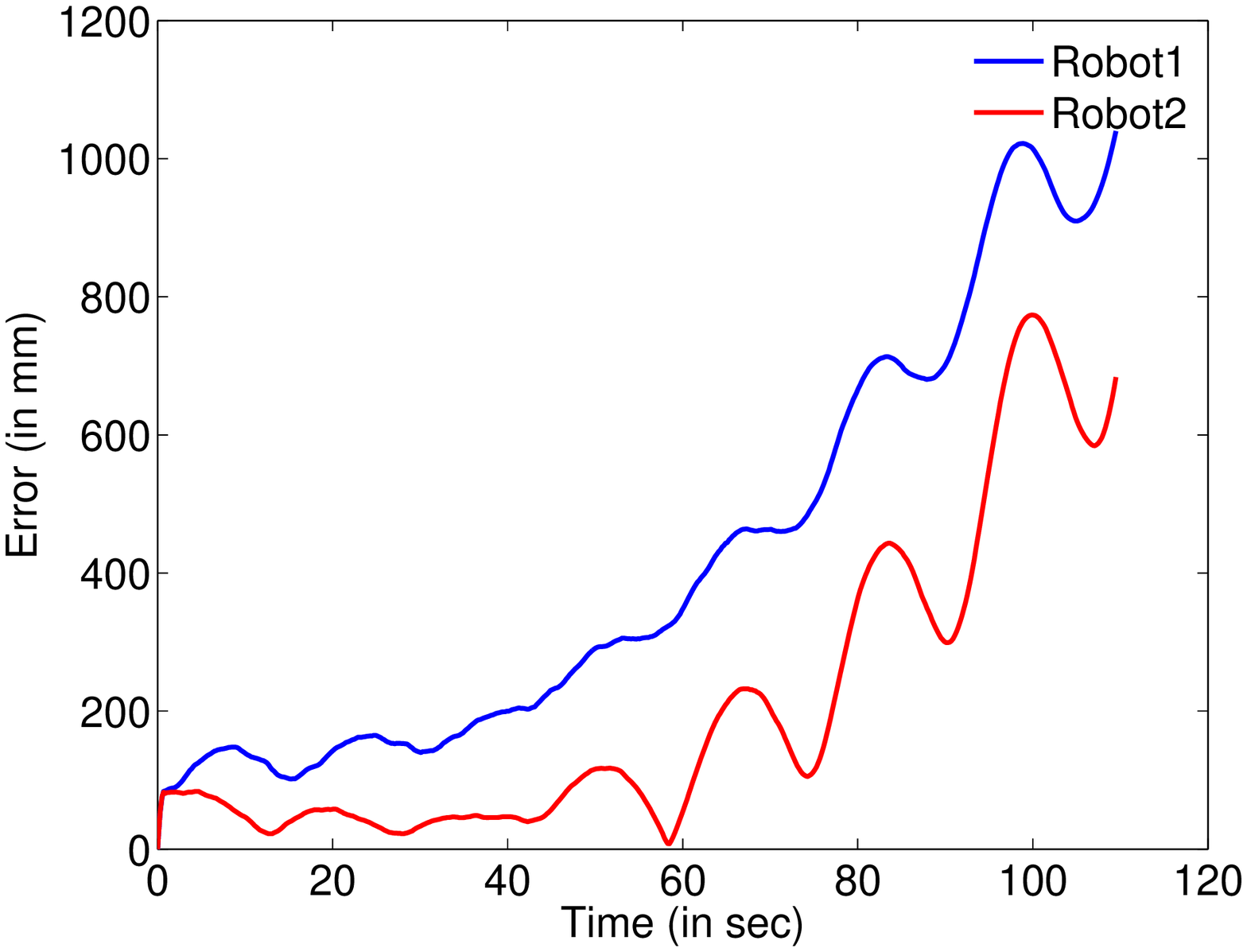}
		\caption{Error}
		\label{fig:mmc_error}
	\end{subfigure}
	\hspace{0.1cm}
	\begin{subfigure}[t]{0.48\textwidth}
		\centering	
		\includegraphics[scale = 0.425, trim = 10 0 40 0, clip = false]{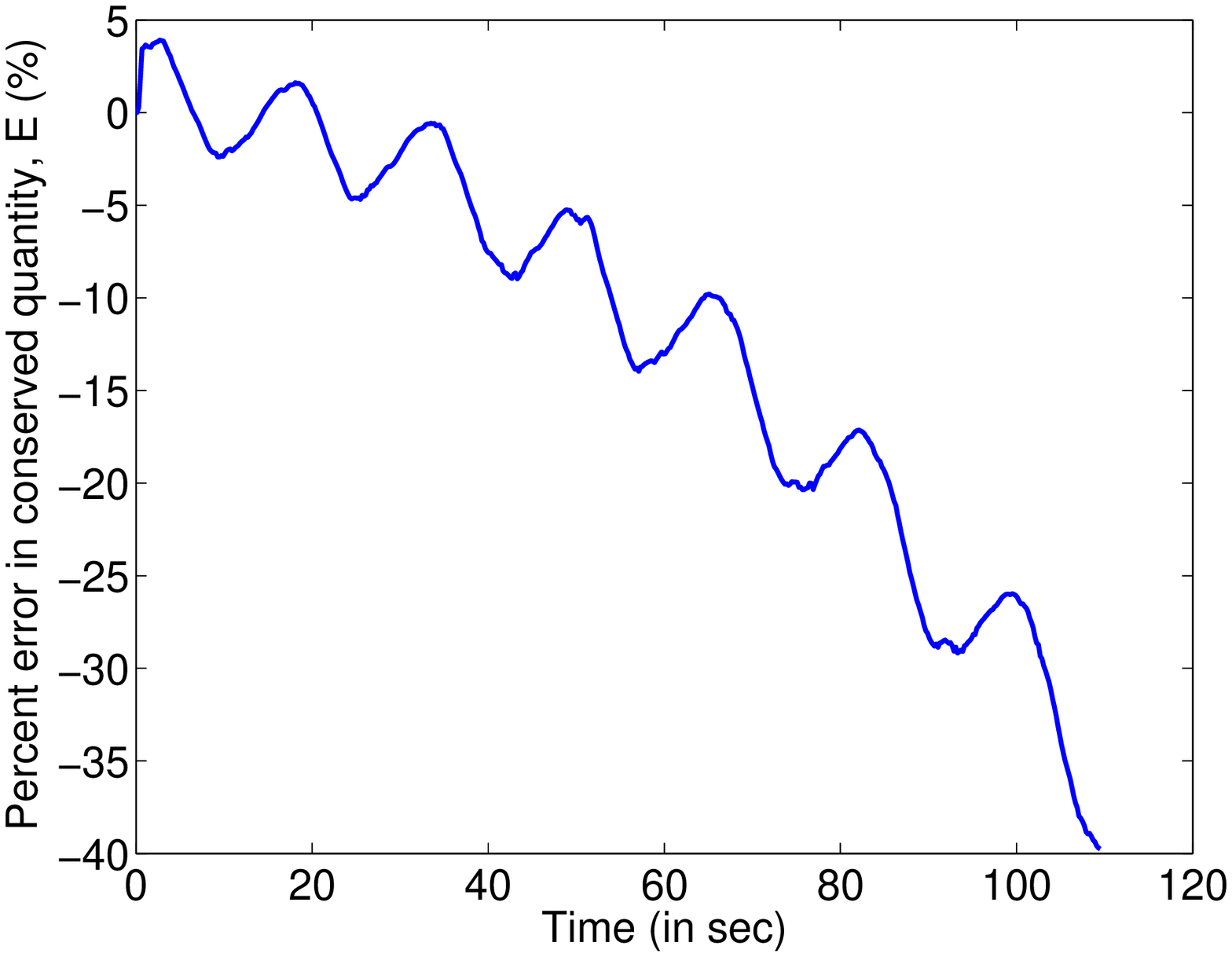}
		\caption{Percentage error in conserved quantity}
		\label{fig:mmc_conserved}
	\end{subfigure}
	\caption{Comparison between experimental and simulation results under pure MMC law in action (with $\mu = 0.001$ mm$^{-1}$ and $\nu_1 = \nu_2 = 200$ mm/sec)  }
	
	\label{fig:mmc_traj}
\end{figure*}

\begin{figure*}[!h]
	\centering
	\begin{subfigure}[t]{0.48\textwidth}
		\centering
		\includegraphics[scale = 0.425, trim = 10 0 40 0, clip = false]{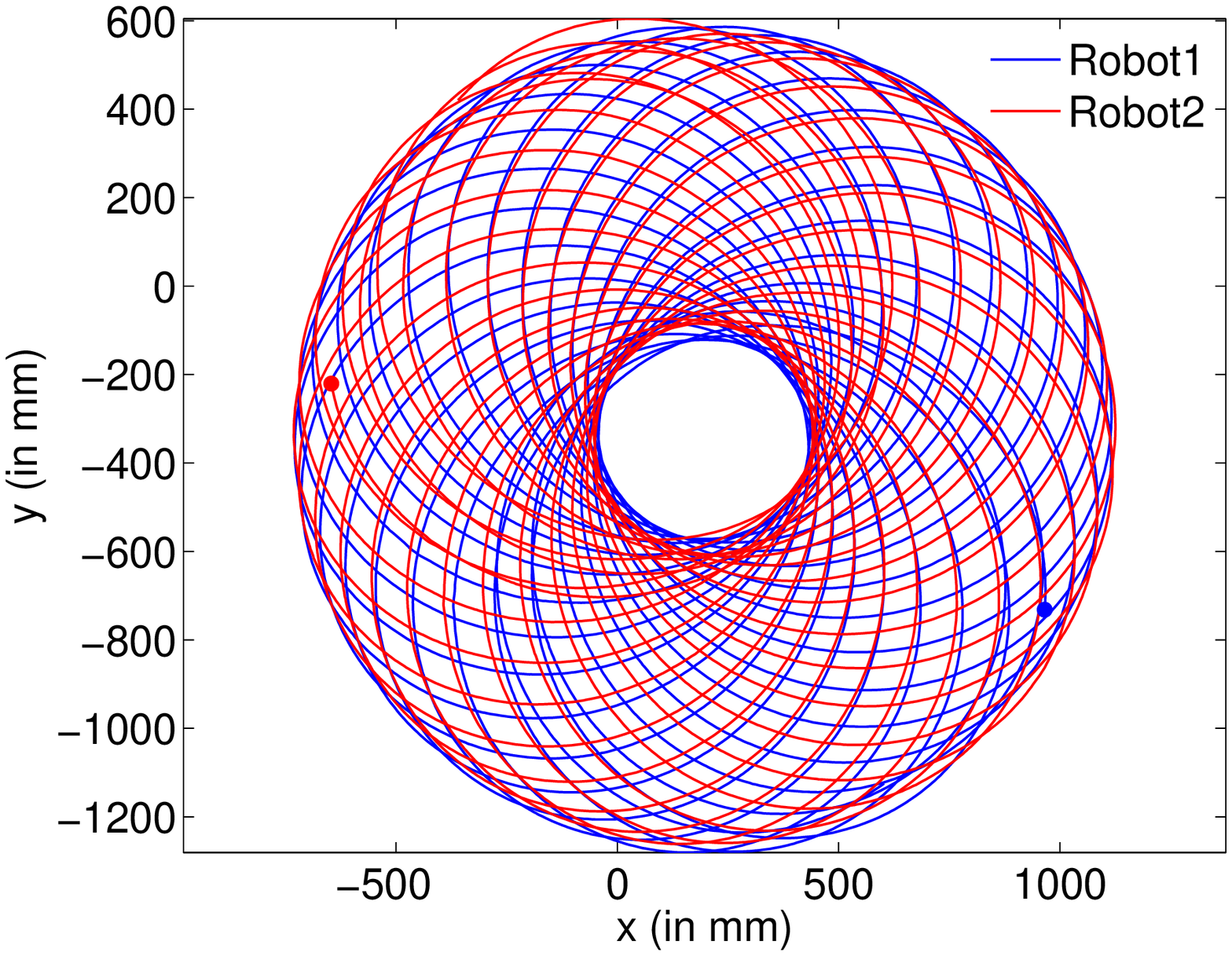}
		\caption{Trajectory (Experiment)}
		\label{fig:mmc_dis_traj}
	\end{subfigure}
	\hspace{0.1cm}
	\begin{subfigure}[t]{0.48\textwidth}
		\centering	
		\includegraphics[scale = 0.425, trim = 10 0 40 0, clip = false]{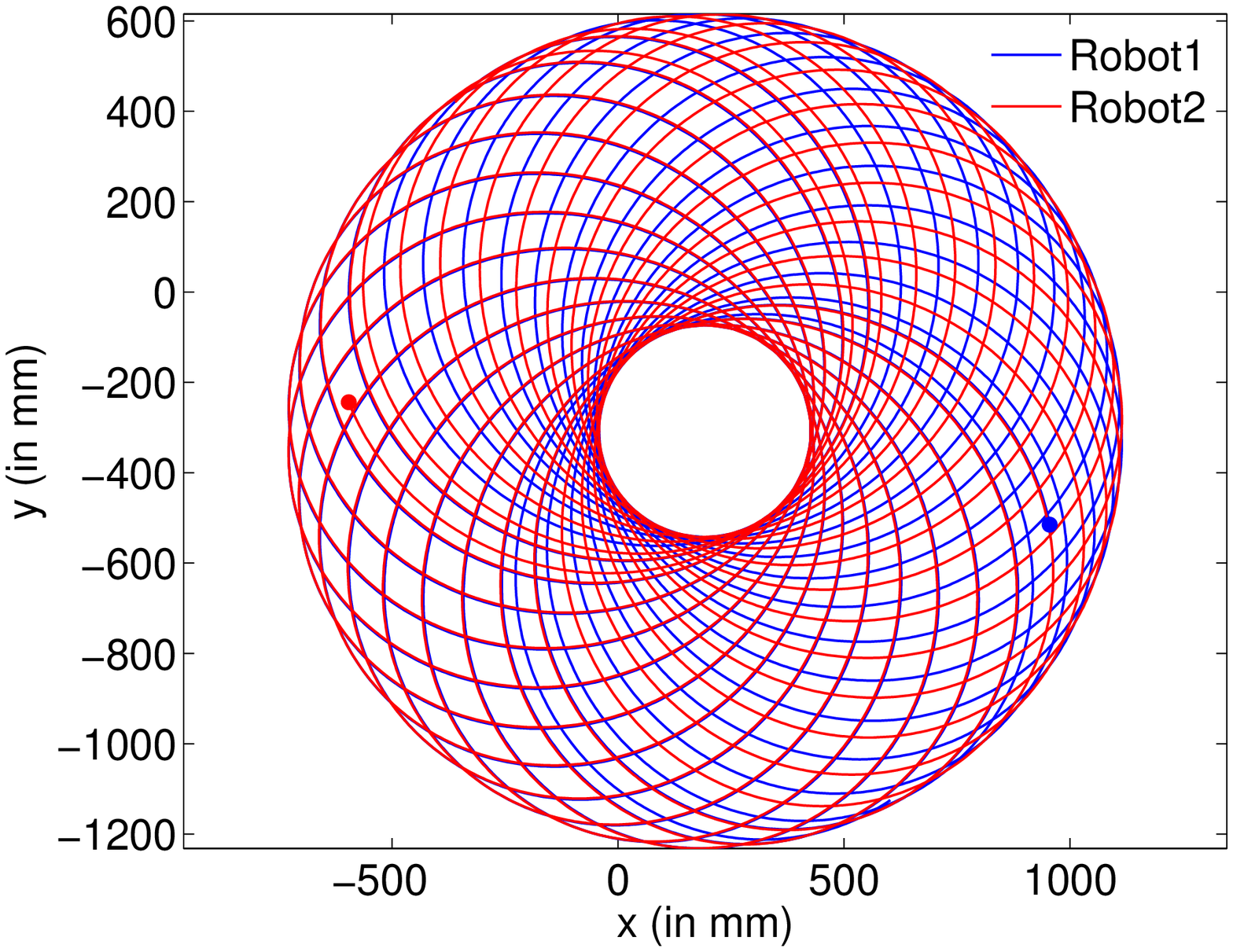}
		\subcaption{Trajectory (Simulation)}
	\end{subfigure}
	\hspace{0.1cm}
	\begin{subfigure}[t]{0.48\textwidth}
		\centering		
		\includegraphics[scale = 0.425, trim = 10 0 40 0, clip = false]{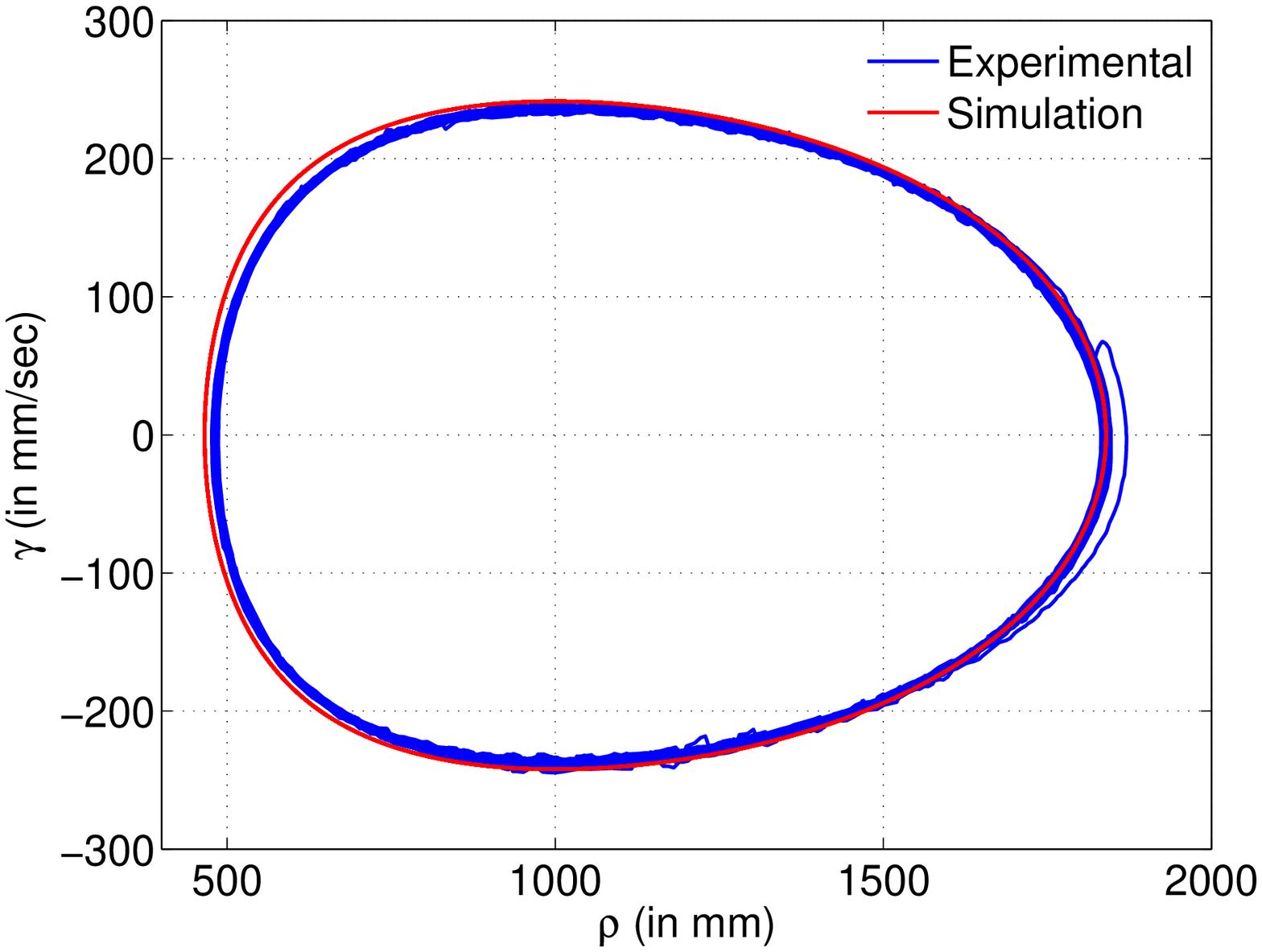}
		\caption{Phase plot}
		\label{fig:mmc_dis_phase}
	\end{subfigure}
	\hspace{0.1cm}
	\begin{subfigure}[t]{0.48\textwidth}
		\centering		
		\includegraphics[scale = 0.425, trim = 10 0 40 0, clip = false]{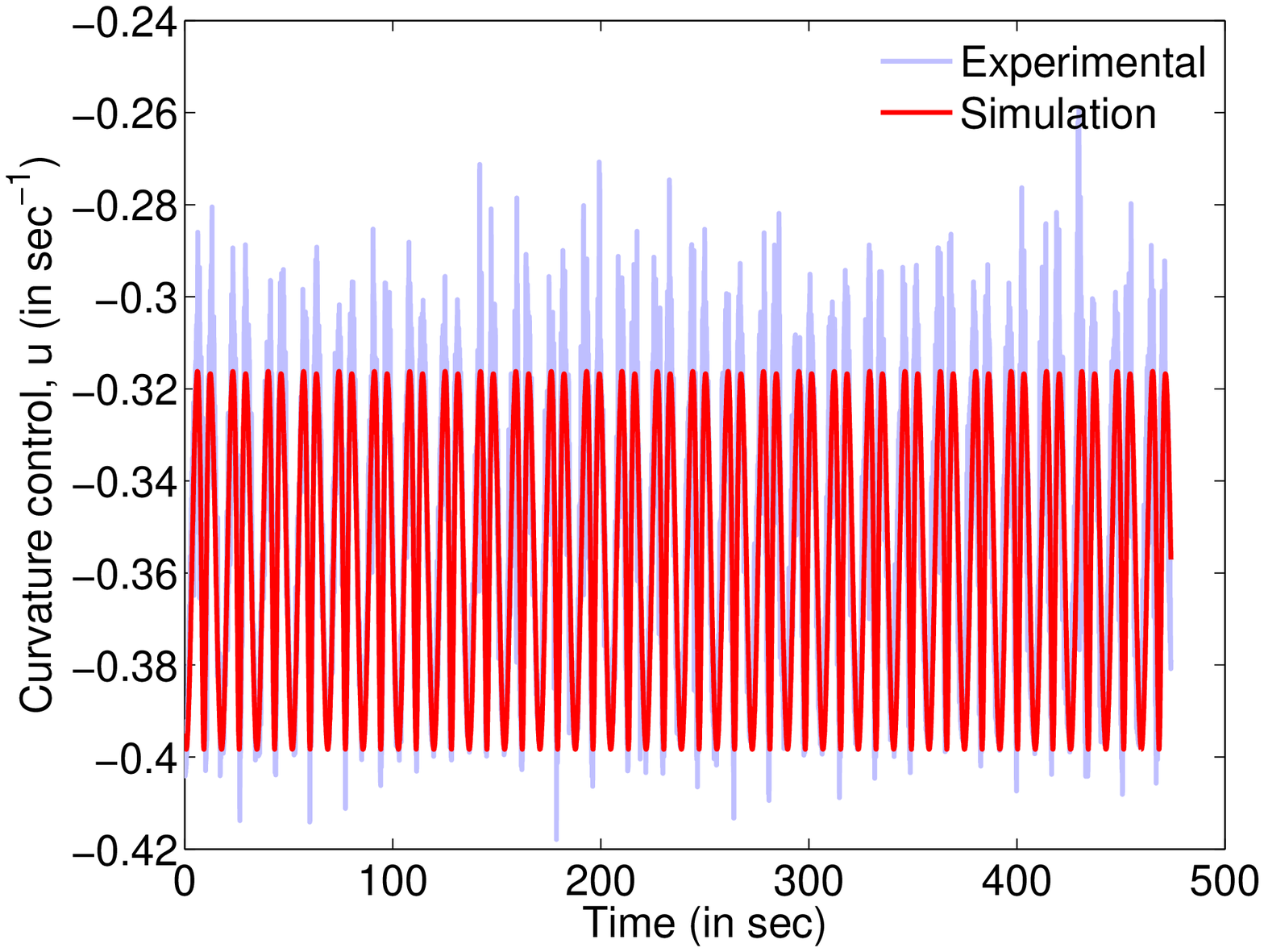}
		\caption{Curvature}
		\label{fig:mmc_dis_curvature}
	\end{subfigure}
	\hspace{0.1cm}
	\begin{subfigure}[t]{0.48\textwidth}
		\centering	
		\includegraphics[scale = 0.425, trim = 10 0 40 0, clip = false]{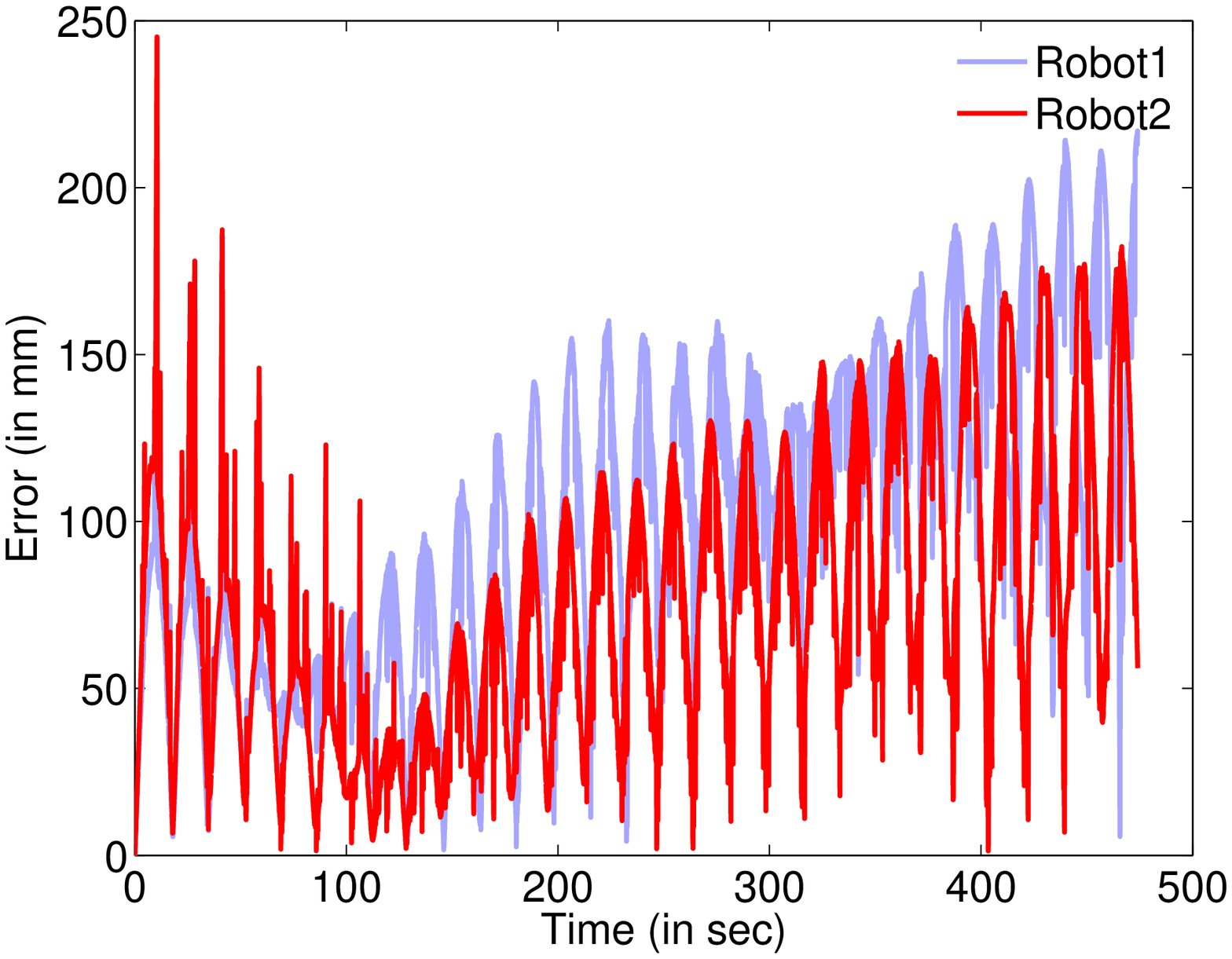}
		\caption{Error}
		\label{fig:mmc_dis_error}
	\end{subfigure}
	\hspace{0.1cm}
	\begin{subfigure}[t]{0.48\textwidth}
		\centering		
		\includegraphics[scale = 0.425, trim = 10 0 40 0, clip = false]{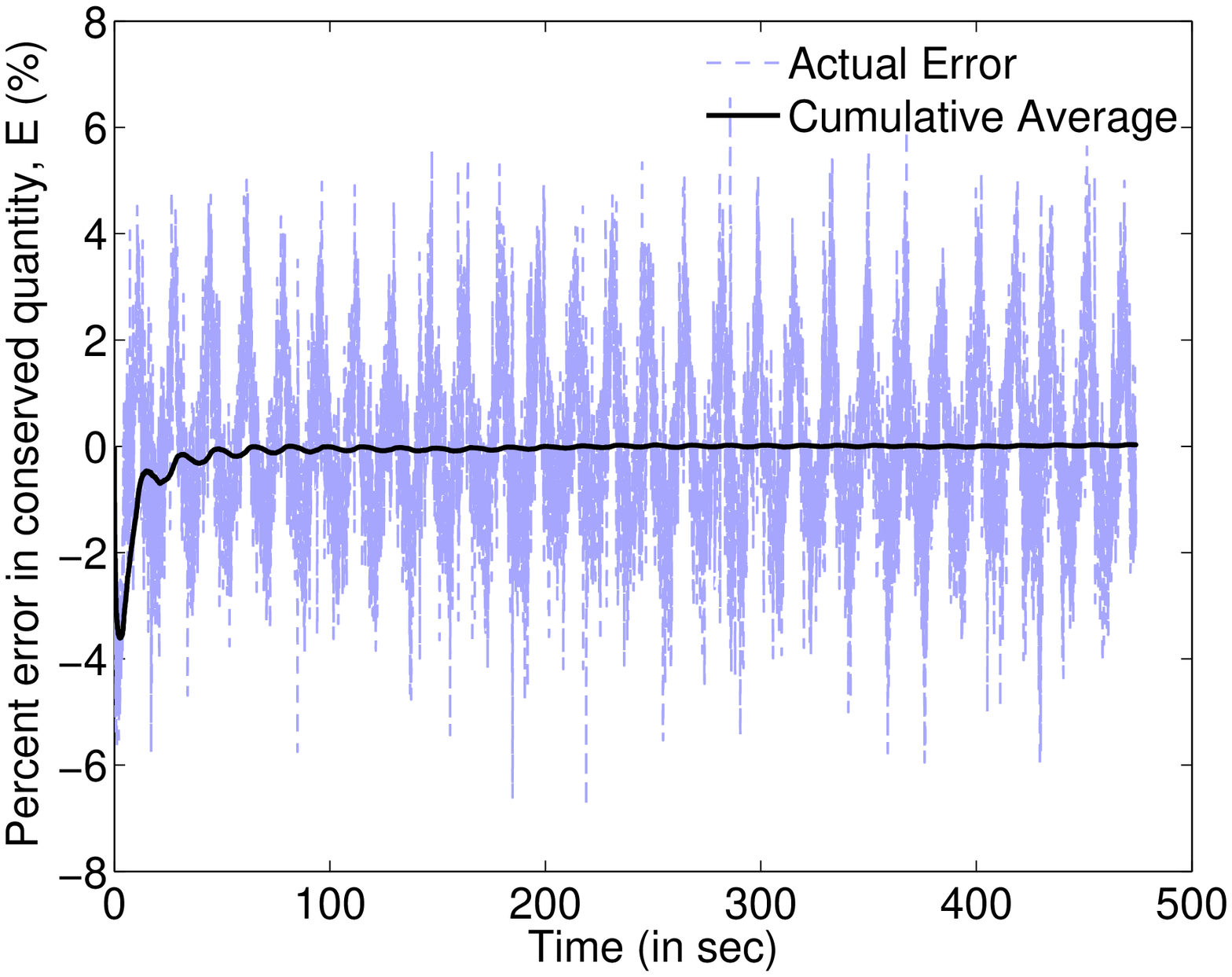}
		\caption{Percentage error in conserved quantity}
		\label{fig:mmc_dis_conserved}
	\end{subfigure}
	\caption{Performance of the system significantly improved upon addition of the dissipative control term (with $\mu = 0.001$ mm$^{-1}$, $\nu_1 = \nu_2 = 200$ mm/sec and $k_d = 1 \times 10^{-15}$ mm$^{-6}$sec$^3$)  }
	\label{fig:mmc_dis}
\end{figure*}
%

Here we will describe the implementation results of MMC, difficulties of pure MMC law (\ref{eq:curvature_MMC}) and usefulness of using the dissipative control term (\ref{eq:dissipation}). First, the original MMC law (\ref{eq:curvature_MMC}) is used to control the robots. The observed trajectories are then compared with theoretically predicted ones. To generate the ideal trajectories in discrete time, we integrate the reduced system dynamics (\ref{eq:rho-gamma}). Considering the conserved quantity in the system (\ref{eq:conserved}), we used the method described in \cite{austin1993almost} for integration instead of general ODE solver, which otherwise would not be able to keep the quantity $E(\rho,\gamma)$ constant to its initial value. From the integrated values of $\rho^k$ and $\gamma^k$, we reconstruct the trajectories (i.e. $\mathbf{r_i}^k, \mathbf{x_i}^k, \mathbf{y_i}^k$ for $i = 1,2$) \cite{mischiati2012dynamics}, where $k$ denotes the time index. At each time instance $t^k$, the error ($e_i^k$) is then calculated as: $e_i^k = |\mathbf{r}_{i,expt}^k - \mathbf{r}_{i,ideal}^k|$. \par 

The plots of a sample run using pure MMC law \eqref{eq:curvature_MMC} are given in Fig \ref{fig:mmc_traj}. The parameters for this run are, $\mu = 0.001$ mm$^{-1}$ and $\nu_1 = \nu_2 = 200$ mm/sec. It can be seen from the figures that the experimental trajectories diverge from the ideal ones and thus the error also keeps increasing over time in Fig \ref{fig:mmc_error}. This also affects the phase plot in Fig \ref{fig:mmc_phase} as we can clearly see the spiraling out type behavior as opposed to the ideal periodic behavior. Possible reasons behind this have already been discussed in section \ref{sec:theory_MMC_dissipation}. \par
%
%
\begin{figure*}[!htp]
	\centering
	\begin{subfigure}[t]{0.48\textwidth}
		\centering
		\includegraphics[scale = 0.425, trim = 10 0 40 0, clip = false]{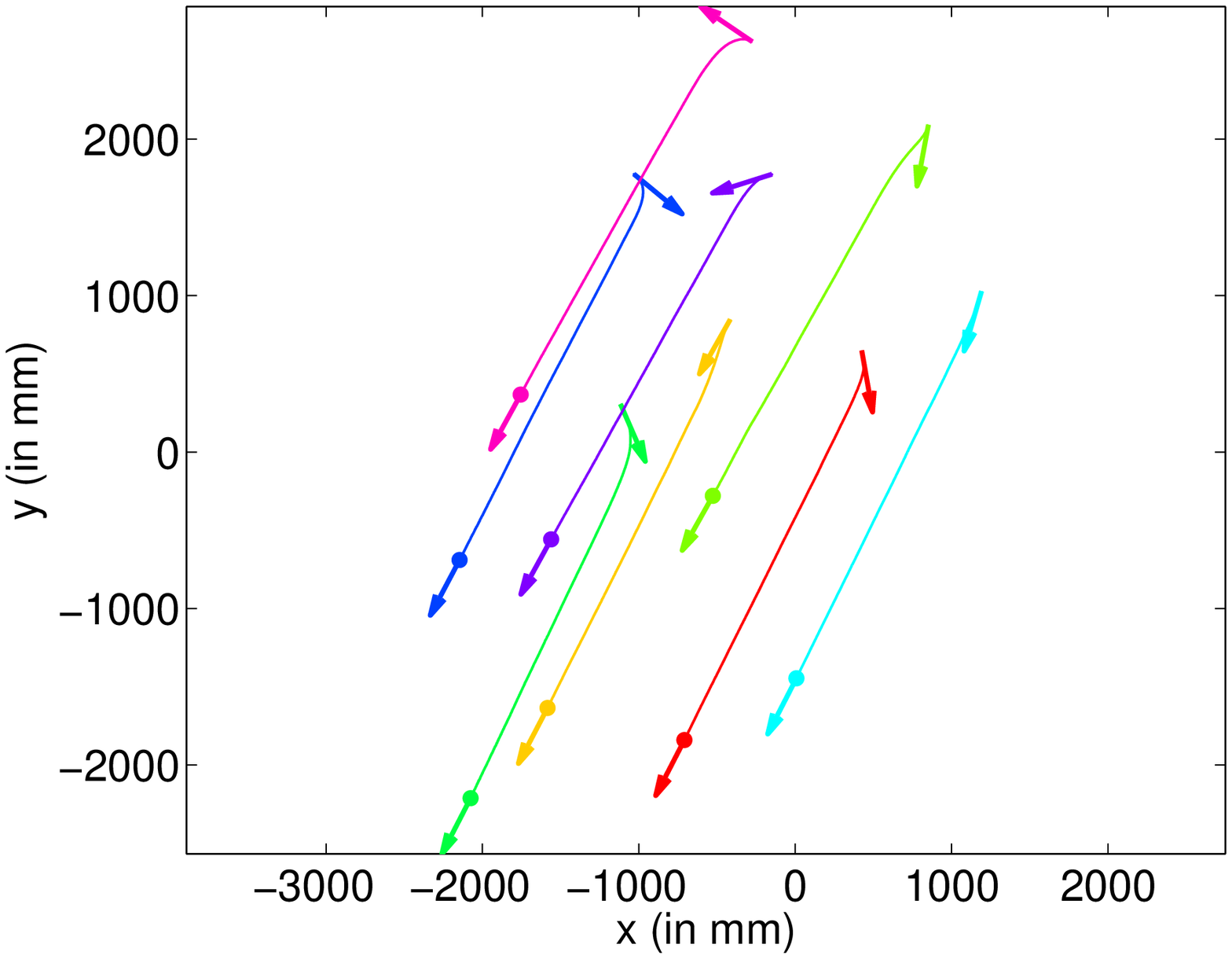}
		\caption{Trajectories: 8 agents, flocking}		
		\label{fig:tva_traj_1}	
	\end{subfigure}
	\hspace{0.1cm}
	\begin{subfigure}[t]{0.48\textwidth}
		\centering
		\begin{scriptsize}
			\psfrag{Cost}[][]{$\Theta(t)$}
			\includegraphics[scale = 0.425, trim = 10 0 40 0, clip = false]{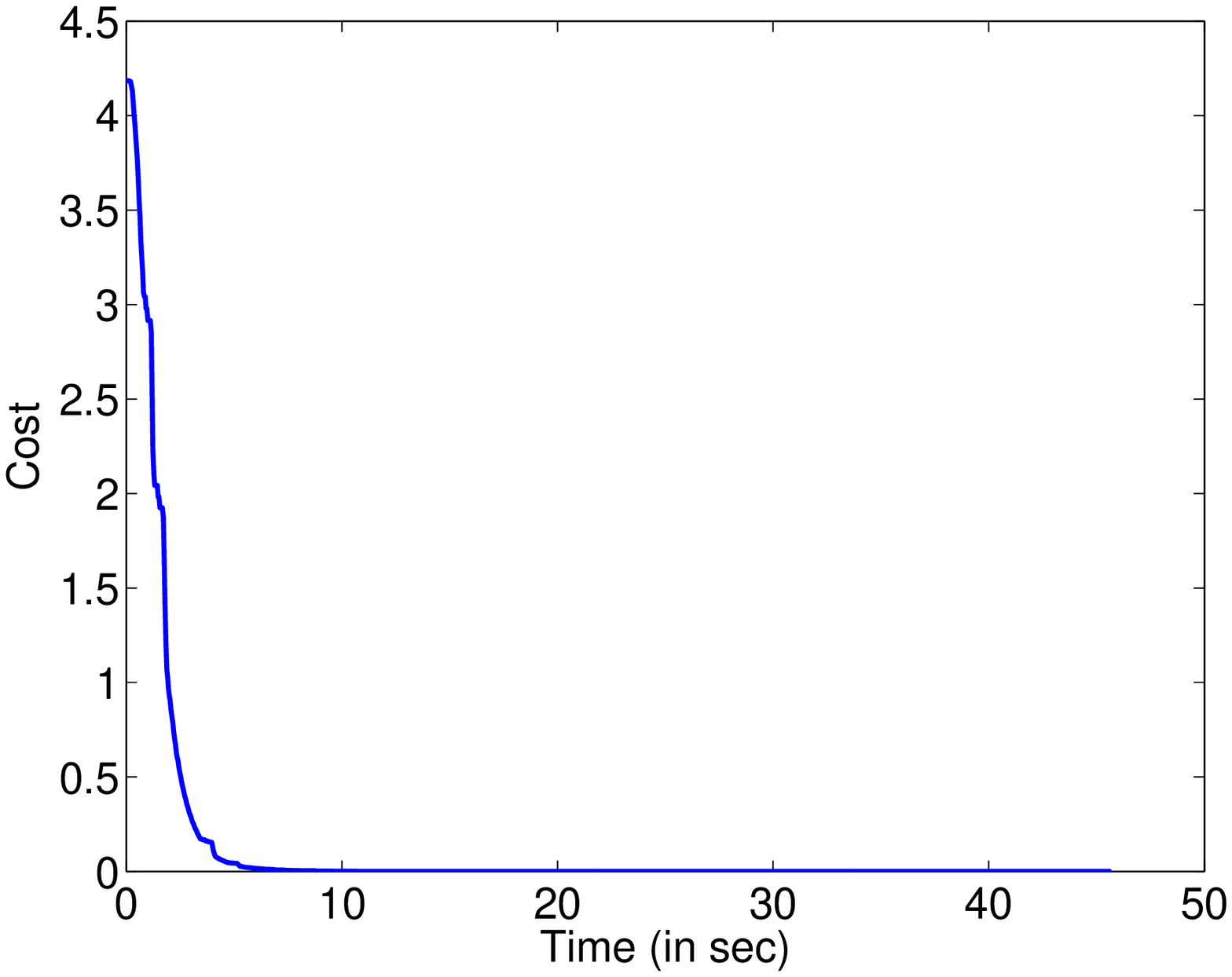}
		\end{scriptsize}
		\caption{Contrast function: 8 agents, flocking}
		\label{fig:cost_1}
	\end{subfigure}
	\hspace{0.1cm}
	\begin{subfigure}[t]{0.48\textwidth}
		\centering
		\includegraphics[scale = 0.425, trim = 10 0 40 0, clip = false]{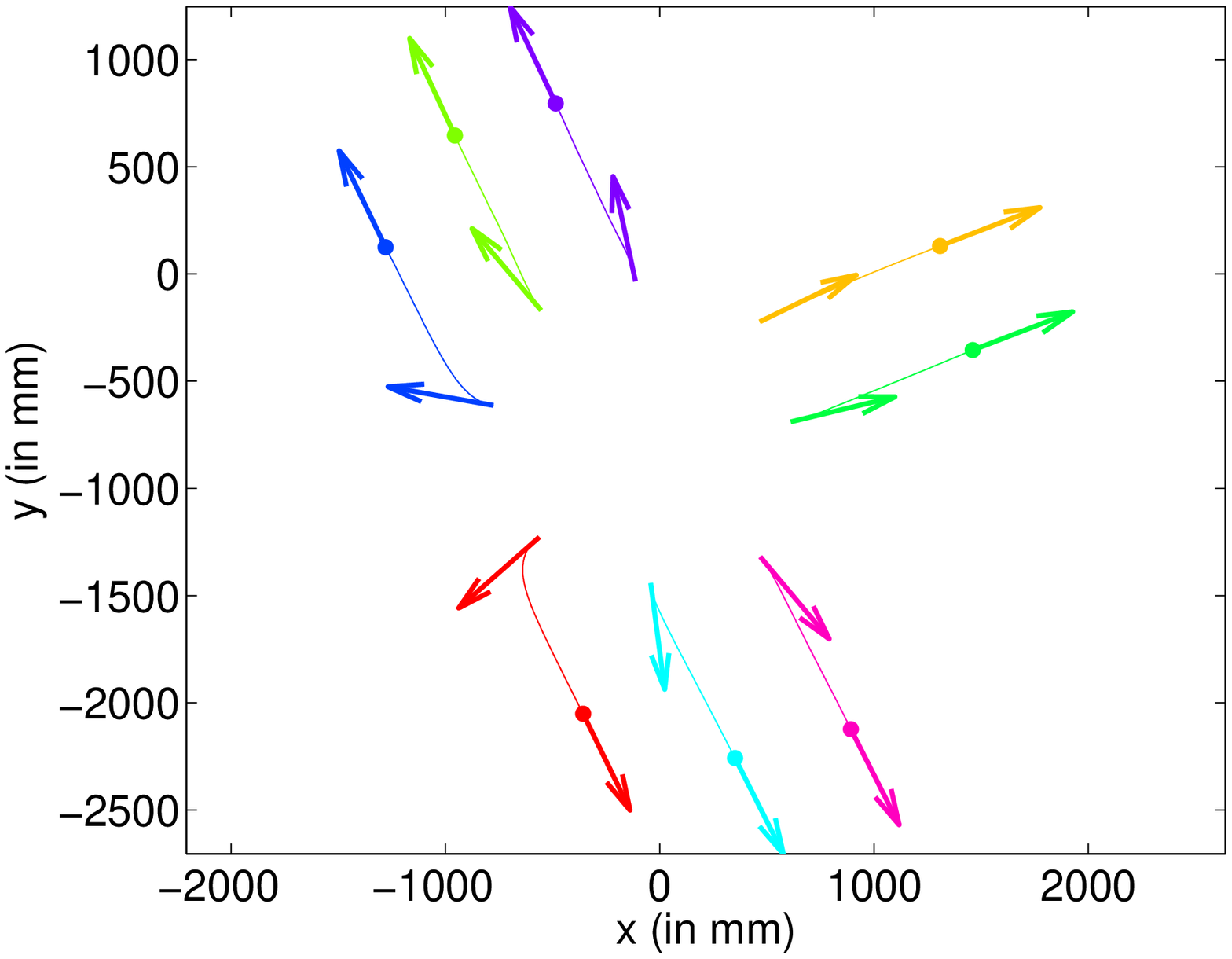}
		\caption{Trajectories: 8 agents, splitting}
		\label{fig:tva_traj_2}
	\end{subfigure}
	\hspace{0.1cm}
	\begin{subfigure}[t]{0.48\textwidth}
		\centering
		\begin{scriptsize}
			\psfrag{Cost}[][]{$\Theta(t)$}
			\includegraphics[scale = 0.425, trim = 10 0 40 0, clip = false]{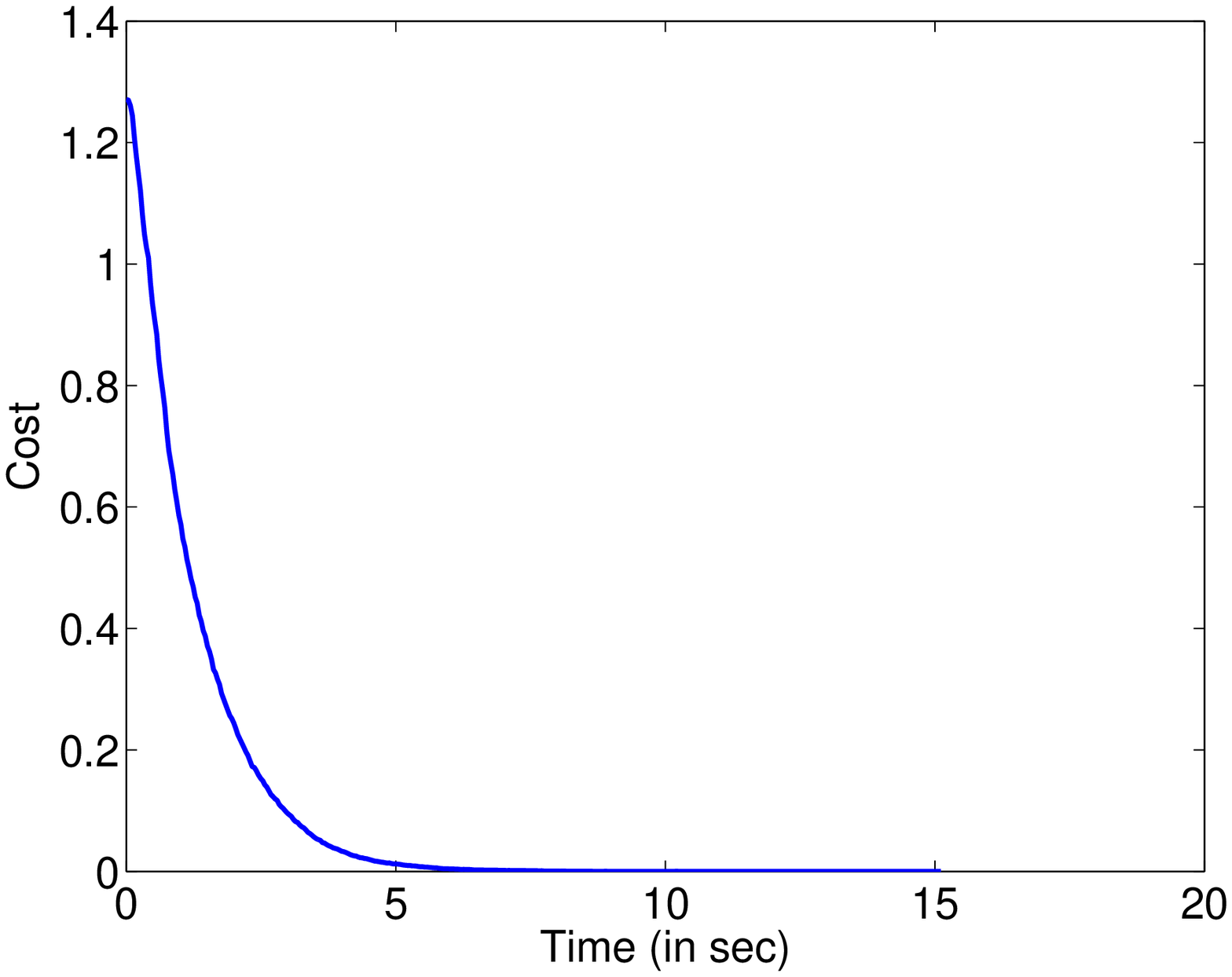}	
		\end{scriptsize}
		\caption{Contrast function: 8 agents, splitting}	
		\label{fig:cost_2}
	\end{subfigure}
	\hspace{0.1cm}
	\begin{subfigure}[t]{0.48\textwidth}
		\centering
		\includegraphics[scale = 0.425, trim = 10 0 40 0, clip = false]{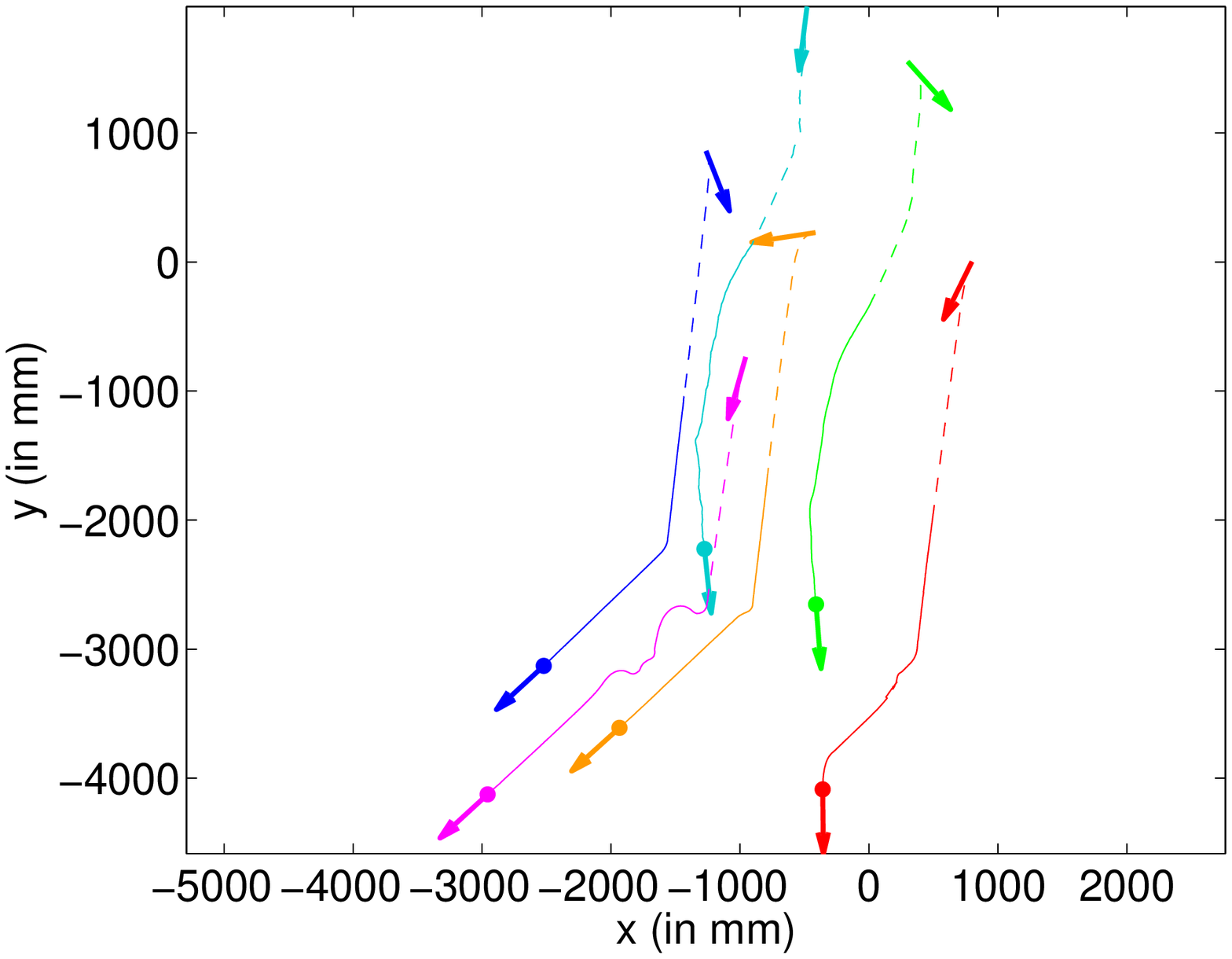}
		\caption{Trajectories: 6 agents, perturbation}		
		\label{fig:tva_traj_3}
	\end{subfigure}
	\hspace{0.1cm}
	\begin{subfigure}[t]{0.48\textwidth}
		\centering
		\begin{scriptsize}
			\psfrag{Cost}[][]{$\Theta(t)$}
			\includegraphics[scale = 0.425, trim = 10 0 40 0, clip = false]{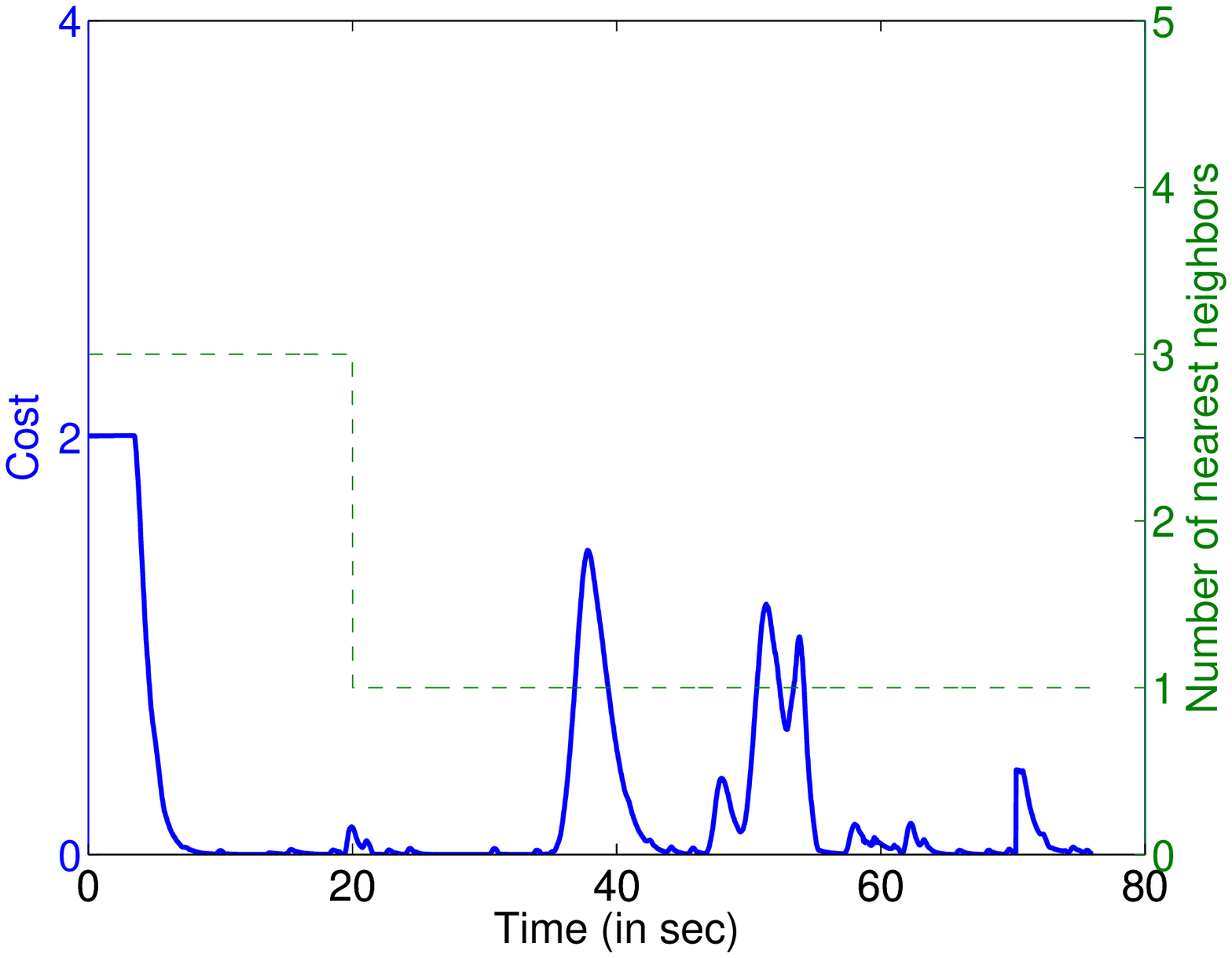}
		\end{scriptsize}
		\caption{Contrast function: 8 agents, perturbation}
		\label{fig:cost_3}
	\end{subfigure}	
	\caption{Trajectory and contrast functions of TVA for (i) Experiment 1 (a),(b) with 8 agents, demonstrates flocking behavior ($K = 3$); (ii) Experiment 2 (c),(d) with 8 agents, describes the splitting behavior due to low neighborhood size ($K = 1$), and (iii) Experiment 3 (e),(f) with 6 agents, shows perturbation can cause a swarm to split, the trajectory of the perturbing agent is not shown.  ($\mu = 1$ Hz and $\nu_i = 60$ mm/sec is kept fixed for all experiments) }
	\label{fig:tva}
\end{figure*}
%

To overcome the problem with original MMC law, a stabilization term has been introduced in the control law (\ref{eq:dissipation}). We applied the modified control law with $k_d = 1 \times 10^{-15}$ mm$^{-6}$sec$^3$ and kept all other parameters fixed. The resulting performance is quite satisfactory as shown in Fig. \ref{fig:mmc_dis_traj} (refer \cite{ISL_Youtube} for implementation video). The error (Fig \ref{fig:mmc_dis_error}) is also bounded ($\sim 250$ mm) within the size of the robots ($\sim 400$ mm). Also, it can be seen from Fig \ref{fig:mmc_dis_conserved} that the cumulative average of the error in conserved quantity converges to almost 0$\%$, indicating superior performance.

\subsection{Implementation of TVA}

We implemented the TVA control law (\ref{TVA_Rule}) in a 2 dimensional setting (i.e. $v_i(t)$ is ignored). As the implementation is in discrete time, we followed Algorithm \ref{TVA_Algo} in our implementation in order to avoid the singular case of $|\mathbf{v}_{_{COM}}| = 0 $. To demonstrate the performance of TVA control law, we designed three different experiments (refer \cite{ISL_Youtube} for implementation videos). In these experiments, the sonar sensors on the robots were activated to sense any obstacle in the direction of motion of the robots and if any robot can sense such an obstacle, it will simply apply a maximum \textit{turning rate} ($\omega^{sat}$) to avoid collision. The sonars are programmed to detect an obstacle only in close proximity ($\sim 300$ mm) of the robots. In all our experiments $\omega^{sat}$ is taken to be 50 rad/sec, forward speeds of all of the robots are kept constant at 60 mm/sec and the value of the parameter $\mu$ is chosen to be 1 Hz. 

\subsubsection{Experiment 1}
A system with eight agents is considered and we apply same TVA law to all of them. The neighborhood size is taken to be three (i.e. $K = 3$). The robots are initially placed in arbitrary positions and directions. The footprints of the robots and the corresponding contrast function, $\Theta(t) = \sum_{i} \Theta_i(t)$ is plotted against time in Fig \ref{fig:tva_traj_1}, \ref{fig:cost_1}. The initial and final directions of the robots are shown using arrows and the final positions of the robots are denoted using dots. It can be seen from Fig \ref{fig:cost_1} that the contrast function decays to zero very quickly which indicates perfect velocity alignment within the swarm. 

\subsubsection{Experiment 2}
Next we decreased the neighborhood size and made it one ($K = 1$), so that every robot only `communicates' with its closest neighbor. We chose the initial positions in such a way that they may form sub-clusters instead of moving as a single swarm. This behavior is called `\textit{splitting}' in a swarm. From Fig \ref{fig:tva_traj_2}, we can clearly see that the swarm of eight robots gradually split from each other and form three different clusters. It is to be noted that even if all the agents are not going in the same direction, the contrast function still converges to zero (Fig. \ref{fig:cost_2}). This happens because each of the robots are aligned with their nearest neighbors and hence each of the individual contrast functions ($\Theta_i(t)$) are zero. This experiment may explain the splitting phenomenon observable in nature.

\subsubsection{Experiment 3}
Lastly, we combined the above two experiments, and conducted an experiment using six robots in a swarm and another robot as a predator. A separate computer was used for manual control of the `predator' robot.

At the beginning, neighborhood size is kept at $K = 3$, such that the `communication' graph among the robots stays connected and they move as an entire swarm in a common direction. When the swarm comes close to the predator, the neighborhood size is decreased to one. As we are not using any onboard visual sensing and the sonar sensing is done only in very close region ($\sim 300$ mm), the change in neighborhood size is made manually. From Fig \ref{fig:cost_3}, we can see that the change in neighborhood size takes place at around 20 seconds and we can also see a tiny jump in the contrast function at that time. The predator then slowly approaches to one of the agents in the swarm, which abiding to its collision avoidance rule, turns to avoid the predator. In Fig \ref{fig:tva_traj_3}, the trajectories of the agents are drawn in dashed lines before the occurrence of this event and in solid lines afterwards. The trajectory of the predator robot in not shown in the figure. After creating the initial perturbation, the predator is slowly moved through the swarm causing some subsequent disturbances. These perturbations create a noticeable impact in the swarm. As the attacked agent turns, its neighbor also tries to align itself with that agent and so does its neighbor. This goes on until the communication graph becomes disconnected and a split in the swarm is then observed \cite{ISL_Youtube} just like in Experiment 2. As we can see in Fig \ref{fig:tva_traj_3}, the swarm is divided in two groups after the attack of the predator. The jumps in plot of the contrast function in Fig \ref{fig:cost_3} symbolize the perturbations caused by the external agent. The contrast function value eventually converges to zero after the members are aligned with their neighbors within each subgroup.
\section{Conclusion and Future Work}
\label{sec:5_Conclusion}
%
This paper has introduced a control strategy (TVA) which, based on local information, attempts to align the individual velocities, and the theoretical analysis for a special case of planar two-agent system has been complemented by experiments in a laboratory environment. Additionally, we chose MMC for implementation due to its dynamic coverage property with potential in many domains (resource harvesting, environmental monitoring, search and rescue missions etc.). Moreover, the exploitive nature of MMC (through coverage in a certain annular region) augments the exploratory nature of TVA (through making the agents move in a common direction). Future works will primarily focus on two areas, namely effect of a beacon (influencing both agents) into MMC, and understanding the effect of a covert leader (equipped with extra information) in a flock. Also, this study does not address how the system behavior will change in a noisy (or uncertain) environment. In future, the authors have plans to extend current results to study the effect of sensor noise (along the lines of \cite{MC_Stochastic}) and perform robustness analysis.
%
%
\section*{Acknowledgments}
The authors would like to thank P. S. Krishnaprasad and K. S. Galloway for their valuable comments and feedback.
\bibliographystyle{IEEEtran}
\bibliography{ICRA_refs}

\end{document}